\documentclass{article}

\usepackage{microtype}
\usepackage{graphicx}
\usepackage{subfigure}
\usepackage{booktabs} 

\usepackage[accepted]{icml2020}

\icmltitlerunning{Model-free RL in Infinite-horizon Average-reward MDPs}

\usepackage[algo2e, linesnumbered, vlined,ruled]{algorithm2e}
\usepackage{amssymb}
\usepackage{amsmath} 
\usepackage{amsfonts} 
\usepackage{bbm} 
\usepackage{amsthm} 
\usepackage{xspace}
\usepackage{times}

\usepackage{graphicx} 
\usepackage{color} 
\usepackage{multirow} 
\usepackage{makecell}

\newtheorem{theorem}{Theorem}
\newtheorem{corollary}{Corollary}[theorem]
\newtheorem{lemma}[theorem]{Lemma}
\newtheorem{definition}{Definition}[section]

\DeclareMathOperator*{\argmin}{argmin}
\DeclareMathOperator*{\argmax}{argmax}
\DeclareMathOperator*{\spn}{sp}

\newcommand{\one}{\mathbf{1}}
\newcommand{\mix}{t_{\text{mix}}}
\newcommand{\calS}{{\mathcal{S}}}
\newcommand{\calA}{{\mathcal{A}}}
\newcommand{\calT}{{\mathcal{T}}}
\newcommand{\calI}{{\mathcal{I}}}

\newcommand{\estimate}{\textsc{EstimateQ}\xspace}
\newcommand{\pg}{\textsc{MDP-OOMD}\xspace}
\newcommand{\oomd}{\textsc{OomdUpdate}\xspace}
\newcommand{\politex}{\textsc{Politex}\xspace}
\newcommand{\expthree}{\textsc{Exp3}\xspace}

\newcommand{\E}{\mathbb{E}}
\newcommand{\hit}{t_{\text{hit}}}
\newcommand{\ratio}{\rho}
\newcommand{\order}{\mathcal{O}}
\newcommand{\otil}{\widetilde{\mathcal{O}}}
\newcommand{\field}[1]{\mathbb{#1}}
\newcommand{\fR}{\field{R}}
\newcommand{\inner}[1]{ \left\langle {#1} \right\rangle }
\newcommand{\inn}[1]{ \langle {#1} \rangle }
\newcommand{\norm}[1]{\left\|{#1}\right\|}

\begin{document}

\twocolumn[
\icmltitle{Model-free Reinforcement Learning in Infinite-horizon \\ Average-reward Markov Decision Processes}



\icmlsetsymbol{equal}{*}

\begin{icmlauthorlist}
\icmlauthor{Chen-Yu Wei\ \ }{}
\icmlauthor{Mehdi Jafarnia-Jahromi\ \ }{}
\icmlauthor{Haipeng Luo\ \ }{}
\icmlauthor{Hiteshi Sharma\ \ }{}
\icmlauthor{Rahul Jain\ \ }{}
\end{icmlauthorlist}
\begin{center}
\{chenyu.wei, mjafarni, haipengl, hiteshis, rahul.jain\}@usc.edu
\end{center}

\icmlcorrespondingauthor{Chen-Yu Wei}{chenyu.wei@usc.edu}
\icmlcorrespondingauthor{Mehdi Jafarnia-Jahromi}{mjafarni@usc.edu}

\icmlkeywords{Reinforcement learning, model-free algorithms, infinite-horizon, average-reward}

\vskip 0.3in
]




\begin{abstract}
\textit{Model-free} reinforcement learning is known to be memory and computation efficient and more amendable to large scale problems. In this paper, two model-free algorithms are introduced for learning \textit{infinite-horizon average-reward} Markov Decision Processes (MDPs).
The first algorithm reduces the problem to the discounted-reward version and achieves $\order(T^{2/3})$ regret after $T$ steps, under the minimal assumption of weakly communicating MDPs.
To our knowledge, this is the first model-free algorithm for general MDPs in this setting.
The second algorithm makes use of recent advances in adaptive algorithms for adversarial multi-armed bandits and improves the regret to $\order(\sqrt{T})$, albeit with a stronger ergodic assumption.
This result significantly improves over the $\order(T^{3/4})$ regret achieved by the only existing model-free algorithm by~\citet{abbasi2019politex} for ergodic MDPs in the infinite-horizon average-reward setting.
\end{abstract}


\section{Introduction}
\label{sec: introduction}

Reinforcement learning (RL) refers to the problem of an agent interacting with an unknown environment with the goal of maximizing its cumulative reward through time.
The environment is usually modeled as a Markov Decision Process (MDP) with an unknown transition kernel and/or an unknown reward function.
The fundamental trade-off between exploration and exploitation is the key challenge for RL: should the agent exploit the available information to optimize the immediate performance, or should it explore the poorly understood states and actions to gather more information to improve future performance?

There are two broad classes of RL algorithms: \textit{model-based} and \textit{model-free}. Model-based algorithms maintain an estimate of the underlying MDP and use that to determine a policy during the learning process. Examples include UCRL2 \citep{jaksch2010near}, REGAL \citep{bartlett2009regal}, PSRL \citep{ouyang2017learning}, SCAL \citep{fruit2018efficient}, UCBVI \citep{azar2017minimax}, EBF \citep{zhang2019regret} and EULER \citep{zanette2019tighter}. Model-based algorithms are well-known for their sample efficiency. However, there are two general disadvantages of model-based algorithms: First, model-based algorithms require large memory to store the estimate of the model parameters.
Second, it is hard to extend model-based approaches to non-parametric settings, e.g., continuous state MDPs.

Model-free algorithms, on the other hand, try to resolve these issues by directly maintaining an estimate of the optimal Q-value function or the optimal policy. Examples include Q-learning \citep{watkins1989learning}, Delayed Q-learning \citep{strehl2006pac}, TRPO \citep{schulman2015trust}, DQN \citep{mnih2013playing}, A3C \citep{mnih2016asynchronous}, and more. 
Model-free algorithms are not only computation and memory efficient, but also easier to be extended to large scale problems by incorporating function approximation.

It was believed that model-free algorithms are less sample-efficient compared to model-based algorithms. However, recently \citet{jin2018q} showed that (model-free) Q-learning algorithm with UCB exploration achieves a nearly-optimal regret bound, implying the possibility of designing algorithms with advantages of both model-free and model-based methods.
\citet{jin2018q} addressed the problem for episodic finite-horizon MDPs. 
Following this work, \citet{dong2019q} extended the result to the infinite-horizon discounted-reward setting. 

\begin{table*}
\caption{Regret comparisons for RL algorithms in infinite-horizon average-reward MDPs with $S$ states, $A$ actions, and $T$ steps. $D$ is the diameter of the MDP, $\spn(v^*)\leq D$ is the span of the optimal value function, $\mathbb{V}_{s, a}^\star := \text{Var}_{s' \sim p(\cdot | s, a) }[v^*(s')] \leq \spn(v^*)^2$ is the variance of the optimal value function, $\mix$ is the mixing time (Def~\ref{def:mixing}), $\hit$ is the hitting time (Def~\ref{def:hitting}), and $\ratio \leq \hit$ is some distribution mismatch coefficient (Eq.~\eqref{eqn:ratio}). 
For more concrete definition of these parameters, see Sections~\ref{sec: preliminaries}-\ref{section: policy optimization}.}
\label{tab: regret comparison}
\begin{center}
\renewcommand{\arraystretch}{1.3}
\begin{tabular}{ |c|l|l|l| }
\hline
 & \textbf{Algorithm} & \textbf{Regret} & \textbf{Comment} \\ \hline
\multirow{8}{*}{Model-based} & REGAL {\small\citep{bartlett2009regal}} & $\otil(\spn(v^*)\sqrt{SAT})$ & no efficient implementation \\
 & UCRL2 {\small\citep{jaksch2010near}} & $\otil(DS\sqrt{AT})$ & - \\
 & PSRL {\small\citep{ouyang2017learning}}& $\otil(\spn(v^*)S\sqrt{AT})$ & Bayesian regret \\
 & OSP {\small\citep{ortner2018regret}}& $\otil(\sqrt{\mix SAT})$ & \makecell[l]{ergodic assumption and \\ no efficient implementation } \\
 & SCAL {\small\citep{fruit2018efficient}} & $\otil(\spn(v^*)S\sqrt{AT})$ & -\\
 & KL-UCRL {\small\citep{talebi2018variance}} & $\otil(\sqrt{S\sum_{s, a}\mathbb{V}_{s, a}^\star T})$ & -\\
 & UCRL2B {\small\citep{fruit2019improved}} & $\otil(S\sqrt{DAT})$ & -\\
 &  EBF {\small\citep{zhang2019regret}}& $\otil(\sqrt{DSAT})$ & no efficient implementation \\ \hline
 \hline
\multirow{3}{*}{Model-free} 
 & \politex {\small\citep{abbasi2019politex}} & $\mix^3\hit\sqrt{SA}T^\frac{3}{4}$ & ergodic assumption \\
 & Optimistic Q-learning (\textbf{\small this work}) & $\otil(\spn(v^*)(SA)^{\frac{1}{3}}T^{\frac{2}{3}})$ &- \\
 & \pg (\textbf{\small this work}) & $\otil(\sqrt{\mix^3 \ratio AT})$ &ergodic assumption \\ \hline
 \hline
 & lower bound {\small\citep{jaksch2010near}} & $\Omega(\sqrt{DSAT})$ & - \\ \hline
\end{tabular}
\end{center}
\end{table*}

However, Q-learning based model-free algorithms with low regret for \textit{infinite-horizon average-reward} MDPs, an equally heavily-studied setting in the RL literature, 
remains unknown.
Designing such algorithms has proven to be rather challenging since the Q-value function estimate may grow unbounded over time and it is hard to control its magnitude in a way that guarantees efficient learning. Moreover, techniques such as backward induction in the finite-horizon setting or contraction mapping in the infinite-horizon discounted setting can not be applied to the infinite-horizon average-reward setting.

In this paper, we make significant progress in this direction and propose two model-free algorithms for learning infinite-horizon average-reward MDPs. The first algorithm, Optimistic Q-learning (Section \ref{sec: qlearning}), achieves a regret bound of $\otil(T^{2/3})$ with high probability for the broad class of weakly communicating MDPs.\footnote{%
Throughout the paper, we use the notation $\otil(\cdot)$ to suppress log terms.
} 
This is the first model-free algorithm in this setting under only the minimal weakly communicating assumption.
The key idea of this algorithm is to artificially introduce a discount factor for the reward, to avoid the aforementioned unbounded Q-value estimate issue,
and to trade-off this effect with the approximation introduced by the discount factor.
We remark that this is very different from the R-learning algorithm of~\cite{schwartz1993reinforcement}, which is a variant of Q-learning with no discount factor for the infinite-horizon average-reward setting.

The second algorithm, \pg (Section \ref{section: policy optimization}), attains an improved regret bound of $\otil(\sqrt{T})$ for the more restricted class of \textit{ergodic} MDPs. 
This algorithm maintains an instance of a multi-armed bandit algorithm at each state to learn the best action.
Importantly, the multi-armed bandit algorithm needs to ensure several key properties to achieve our claimed regret bound,
and to this end we make use of the recent advances for adaptive adversarial bandit algorithms from \citep{wei2018more} in a novel way.

To the best of our knowledge, 
the only existing model-free algorithm for this setting is the \politex algorithm~\citep{abbasi2019politex, abbasi2019exploration}, which achieves $\otil(T^{3/4})$ regret for ergodic MDPs only. 
Both of our algorithms enjoy a better bound compared to \politex, and the first algorithm even removes the ergodic assumption completely.\footnote{\politex is studied in a more general setup with function approximation though. See the end of Section~\ref{sec:PO} for more comparisons.}

For comparisons with other existing model-based approaches for this problem, see Table~\ref{tab: regret comparison}.
We also conduct experiments comparing our two algorithms.
Details are deferred to Appendix~\ref{app:experiments} due to space constraints.
\section{Related Work}
We review the related literature with regret guarantees for learning MDPs with finite state and action spaces (there are many other works on asymptotic convergence or sample complexity, a different focus compared to our work). Three common settings have been studied: 1) finite-horizon episodic setting, 2) infinite-horizon discounted setting, and 3) infinite-horizon average-reward setting. For the first two settings, previous works have designed efficient algorithms with regret bound or sample complexity that is (almost) information-theoretically optimal, using either model-based approaches such as~\citep{azar2017minimax},
or model-free approaches such as~\citep{jin2018q, dong2019q}.

For the infinite-horizon average-reward setting, 
many model-based algorithms have been proposed, such as~\citep{auer2007logarithmic, jaksch2010near, ouyang2017learning, agrawal2017optimistic, talebi2018variance, fruit2018near, fruit2018efficient}.
These algorithms either conduct \textit{posterior sampling} or follow  the \textit{optimism in face of uncertainty} principle to build an MDP model estimate and then plan according to the estimate (hence model-based). 
They all achieve $\tilde{\mathcal{O}}(\sqrt{T})$ regret, but the dependence on other parameters are suboptimal.
Recent works made progress toward obtaining the optimal bound \citep{ortner2018regret, zhang2019regret}; however, their algorithms are not computationally efficient -- the time complexity scales exponentially in the number of states. 
On the other hand, except for the naive approach of combining Q-learning with $\epsilon$-greedy exploration (which is known to suffer regret exponential in some parameters~\citep{osband2014generalization}),
the only existing model-free algorithm for this setting is \politex, which only works for ergodic MDPs.


Two additional works are closely related to our second algorithm \pg: \cite{neu2013online} and~\cite{wang2017primal}.
They all belong to \textit{policy optimization} method where the learner tries to learn the parameter of the optimal policy directly.
Their settings are quite different from ours and the results are not comparable.
We defer more detailed comparisons with these two works to the end of Section~\ref{sec:PO}.

\section{Preliminaries}
\label{sec: preliminaries}
An infinite-horizon average-reward Markov Decision Process (MDP) can be described by $(\mathcal{S, A}, r, p)$ where $\mathcal{S}$ is the state space, $\mathcal{A}$ is the action space, $r: \mathcal{S \times A} \to [0, 1]$ is the reward function and $p: \mathcal{S}^2 \times \mathcal{A} \to [0, 1]$ is the transition probability such that $p(s'| s, a) := \mathbb{P}(s_{t+1}=s' \mid s_t=s, a_t=a)$ for $s_t \in \mathcal{S}, a_t \in \mathcal{A}$ and $t=1,2,3,\cdots$. We assume that $\mathcal{S}$ and $\mathcal{A}$ are finite sets with cardinalities $S$ and $A$, respectively. The average reward per stage of a deterministic/stationary policy $\pi: \mathcal{S \to A}$ starting from state $s$ is defined as 
\begin{align*}
J^\pi(s) := \liminf_{T \to \infty} \frac{1}{T} \mathbb{E} \left[\sum_{t=1}^{T} r(s_t, \pi(s_t)) \Bigm\vert s_1= s \right]
\end{align*}
where $s_{t+1}$ is drawn from $p(\cdot | s_t, \pi(s_t))$.
Let $J^*(s) := \max_{\pi \in \mathcal{A}^\mathcal{S}} J^\pi(s)$. A policy $\pi^*$ is said to be optimal if it satisfies $J^{\pi^*}(s) = J^*(s)$ for all $s \in \mathcal{S}$. 

We consider two standard classes of MDPs in this paper: (1) \textit{weakly communicating} MDPs defined in Section \ref{sec: qlearning} and (2) \textit{ergodic} MDPs defined in Section \ref{section: policy optimization}.
The weakly communicating assumption is weaker than the ergodic assumption, and is in fact known to be necessary for learning infinite-horizon MDPs with low regret \citep{bartlett2009regal}.

Standard MDP theory \citep{puterman2014markov} shows that for these two classes, there exist $q^*: \mathcal{S \times A} \to \mathbb{R}$ (unique up to an additive constant) and unique $J^* \in [0, 1]$ such that $J^*(s) = J^*$ for all $s \in \mathcal{S}$ and the following Bellman equation holds:
\begin{align}
\label{eq: bellman equation}
J^* + q^*(s, a) = r(s, a) + \mathbb{E}_{s' \sim p(\cdot | s, a)}[v^*(s')],
\end{align}
where $v^*(s) := \max_{a\in\mathcal{A}} q^*(s, a)$. The optimal policy is then obtained by $\pi^*(s) = \argmax_a q^*(s, a)$. 

We consider a learning problem where $\mathcal{S}, \mathcal{A}$ and the reward function $r$ are known to the agent, but not the transition probability $p$ (so one cannot directly solve the Bellman equation). 
The knowledge of the reward function is a typical assumption as in \cite{bartlett2009regal,gopalan2015thompson,ouyang2017learning}, and can be removed at the expense of a constant factor for the regret bound.

Specifically, the learning protocol is as follows. 
An agent starts at an arbitrary state $s_1\in\mathcal{S}$.
At each time step $t= 1, 2, 3, \cdots$, the agent observes state $s_t \in \mathcal{S}$ and takes action $a_t \in \mathcal{A}$ which is a function of the history $s_1, a_1, s_2, a_2,  \cdots, s_{t-1}, a_{t-1}, s_t$. The environment then determines the next state by drawing $s_{t+1}$ according to $p( \cdot | s_t, a_t)$.  
The performance of a learning algorithm is evaluated through the notion of \textit {cumulative regret}, defined as the difference between the total reward of the optimal policy and that of the algorithm:
\begin{align*}
R_T := \sum_{t=1}^T \Big(J^* - r(s_t, a_t)\Big).
\end{align*}
Since $r \in [0, 1]$ (and subsequently $J^* \in [0, 1]$), the regret can at worst grow linearly with $T$. If a learning algorithm achieves sub-linear regret, then $R_T/T$ goes to zero, i.e.,  the average reward of the algorithm converges to the optimal per stage reward $J^*$. 
The best existing regret bound is $\otil(\sqrt{DSAT})$ achieved by a model-based algorithm \citep{zhang2019regret} (where $D$ is the diameter of the MDP) and it matches the lower bound of \cite{jaksch2010near}. 


\section{Optimistic Q-Learning}
\label{sec: qlearning}
In this section, we introduce our first algorithm, \textsc{Optimistic Q-learning} (see Algorithm~\ref{alg: optimistic qlearning} for pseudocode).
The algorithm works for any weakly communicating MDPs.
An MDP is weakly communicating if its state space $\mathcal{S}$ can be partitioned into two subsets: in the first subset, all states are transient under any stationary policy; in the second subset, every two states are accessible from each other under some stationary policy.
It is well-known that the weakly communicating condition is necessary for ensuring low regret in this setting \citep{bartlett2009regal}.

Define $\spn(v^*) = \max_{s} v^*(s) - \min_{s} v^*(s)$  to be the span of the value function, which is known to be bounded for weakly communicating MDPs.
In particular, it is bounded by the diameter of the MDP (see~\citep[Lemma~38.1]{lattimore2018bandit}).
We assume that $\spn(v^*)$ is known and use it to set the parameters. 
However, in the case when it is unknown, we can replace $\spn(v^*)$ with any upper bound of it (e.g. the diameter) in both the algorithm and the analysis.

\setcounter{AlgoLine}{0}
\begin{algorithm}[t]
\caption{\textsc{Optimistic Q-learning}}
\label{alg: optimistic qlearning}
\textbf{Parameters: } $H \geq 2$, confidence level $\delta \in (0,1)$\\
\textbf{Initialization: } 
 $\gamma=1-\frac{1}{H}, \quad \forall s: \hat{V}_1(s)=H$ \\
 $\forall s,a: Q_1(s,a) = \hat{Q}_1(s,a)=H, \;\; n_1(s,a)=0$ \\
\textbf{Define:}  
 $\forall \tau, \alpha_\tau = \frac{H+1}{H+\tau}$, $b_\tau = 4\spn(v^*)\sqrt{\frac{H}{\tau}\ln \frac{2T}{\delta}}$ \\
\For{$t=1, \ldots, T$}{
\nl    Take action $a_t = \argmax_{a\in\calA} \hat{Q}_{t}(s_t, a)$. \label{line:greedy} \\
\nl    Observe $s_{t+1}$. \\
\nl    Update: 
    \begin{align}
    n_{t+1}(s_t, a_t) &\leftarrow n_{t}(s_t, a_t)+1 \notag\\
    \tau &\leftarrow n_{t+1}(s_t, a_t)  \notag \\
    Q_{t+1}(s_t,a_t) &\leftarrow (1-\alpha_\tau)Q_{t}(s_t,a_t) \notag \\
      + \alpha_\tau & \left[r(s_t,a_t) + \gamma \hat{V}_{t}(s_{t+1}) + b_\tau\right]  \label{eqn: q update}\\
     \hat{Q}_{t+1}(s_t,a_t) &\leftarrow \min\left\{ \hat{Q}_{t}(s_t,a_t), Q_{t+1}(s_t,a_t) \right\} \notag \\
     \hat{V}_{t+1}(s_t)&\leftarrow \max_{a\in\calA}\hat{Q}_{t+1}(s_t, a) . \notag
    \end{align}
    (All other entries of $n_{t+1}, Q_{t+1}, \hat{Q}_{t+1}, \hat{V}_{t+1}$ remain the same as those in $n_{t}, Q_t, \hat{Q}_{t}, \hat{V}_{t}$.) 
} 
\end{algorithm}

The key idea of Algorithm~\ref{alg: optimistic qlearning} is to solve the undiscounted problem via learning a discounted MDP (with the same states, actions, reward function, and transition), for some discount factor $\gamma$ (defined in terms of a parameter $H$).
Define $V^*$ and $Q^*$ to be the optimal value-function and Q-function of the discounted MDP, satisfying the Bellman equation:
\begin{align*}
\forall (s, a), \quad Q^*(s, a) &= r(s, a) + \gamma\mathbb{E}_{s' \sim p(\cdot | s, a)}[V^*(s')] \\
\forall s, \quad\quad V^*(s) &= \max_{a\in\mathcal{A}} Q^*(s,a).
\end{align*}

The way we learn this discounted MDP is essentially the same as the algorithm of~\citet{dong2019q}, which itself is based on the idea from~\cite{jin2018q}.
Specifically, the algorithm maintains an estimate $\hat{V}_t$ for the optimal value function $V^*$ and $\hat{Q}_t$ for the optimal Q-function $Q^*$, which itself is a clipped version of another estimate $Q_t$.
Each time the algorithm takes a greedy action with the maximum estimated Q value (Line~\ref{line:greedy}).
After seeing the next state, the algorithm makes a stochastic update of $Q_t$ based on the Bellman equation, importantly with an extra {\it bonus} term $b_\tau$ and a carefully chosen step size $\alpha_\tau$ (Eq.\eqref{eqn: q update}).
Here, $\tau$ is the number of times the current state-action pair has been visited, and the bonus term $b_\tau$ scales as $\order(\sqrt{H/\tau})$, which encourages exploration since it shrinks every time a state-action pair is executed.
The choice of the step size $\alpha_\tau$ is also crucial as pointed out in~\cite{jin2018q} and determines a certain effective period of the history for the current update.

While the algorithmic idea is similar to~\cite{dong2019q},
we emphasize that our analysis is different and novel:
\begin{itemize}
\item
First, \citet{dong2019q} analyze the sample complexity of their algorithm while we analyze the regret.

\item
Second, we need to deal with the approximation effect due to the difference between the discounted MDP and the original undiscounted one (Lemma~\ref{lem: discounted average connection}).

\item
Finally, part of our analysis improves over that of \cite{dong2019q} (specifically our Lemma~\ref{lem: v-star minus q-star}). 
Following the original analysis of~\citep{dong2019q} would lead to a worse bound here.
\end{itemize}

We now state the main regret guarantee of Algorithm~\ref{alg: optimistic qlearning}.

\begin{theorem}
\label{thm: q-learning}
If the MDP is weakly communicating, Algorithm \ref{alg: optimistic qlearning} with $H = \min\left\{\sqrt{\frac{\spn(v^*)T}{SA}}, \left(\frac{T}{SA\ln\frac{4T}{\delta}}\right)^{\frac{1}{3}}\right\}$ ensures that with probability at least $1 - \delta$, $R_T$ is of order
\begin{align*}
\order\left(\sqrt{\spn(v^*)SAT} + \spn(v^*)\left(T^\frac{2}{3}\left(SA\ln \tfrac{T}{\delta}\right)^\frac{1}{3} + \sqrt{T\ln\tfrac{1}{\delta}}\right)\right).
\end{align*}
\end{theorem}

Our regret bound scales as $\otil(T^{2/3})$ and is suboptimal compared to model-based approaches with $\otil(\sqrt{T})$ regret (such as UCRL2) that matches the information-theoretic lower bound~\citep{jaksch2010near}.
However, this is the first model-free algorithm with sub-linear regret (under only the weakly communicating condition), and how to achieve $\otil(\sqrt{T})$ regret via model-free algorithms remains unknown.
Also note that our bound depends on $\spn(v^*)$ instead of the potentially much larger diameter of the MDP.
 To our knowledge, existing approaches that achieve $\spn(v^*)$ dependence are all model-based~\citep{bartlett2009regal, ouyang2017learning, fruit2018efficient} and use very different arguments. 

\subsection{Proof sketch of Theorem \ref{thm: q-learning}}
The proof starts by decomposing the regret as
\begin{align*}
R_T &= \sum_{t=1}^T \left(J^* - r(s_t, a_t)\right) \\
&\begin{aligned}
= \sum_{t=1}^T &\left(J^* - (1 - \gamma)V^*(s_t) \right) \\
&+ \sum_{t=1}^T \left(V^*(s_t) - Q^*(s_t, a_t)\right) \\
&+ \sum_{t=1}^T \left(Q^*(s_t, a_t) - \gamma V^*(s_t) - r(s_t, a_t)\right).
\end{aligned} 
\end{align*}
Each of these three terms are handled through Lemmas \ref{lem: discounted average connection}, \ref{lem: v-star minus q-star} and \ref{lem: discounted bellman} whose proofs are deferred to the appendix. 
Plugging in $\gamma = 1 - \frac{1}{H}$ and picking the optimal $H$ finish the proof.
One can see that the $\otil(T^{2/3})$ regret comes from the bound $\frac{T}{H}$ from the first term and the bound $\sqrt{HT}$ from the second.


\begin{lemma}
\label{lem: discounted average connection}
The optimal value function $V^*$ of the discounted MDP satisfies
\begin{enumerate}
    \item $|J^* - (1-\gamma) V^*(s)| \leq (1-\gamma)\spn(v^*)$, $\forall s\in\calS$,
    \item $\spn(V^*) \leq 2\spn(v^*)$.
\end{enumerate}
\end{lemma}
This lemma shows that the difference between the optimal value in the discounted setting (scaled by $1-\gamma$) and that of the undiscounted setting is small as long as $\gamma$ is close to $1$.
The proof is by combining the Bellman equation of the these two settings and direct calculations.

\begin{lemma}
\label{lem: v-star minus q-star}
With probability at least $1 - \delta$, we have
\begin{align*}
    &\sum_{t=1}^T \left(V^*(s_t)- Q^*(s_t,a_t)\right) \\
    &\leq 4HSA + 24\spn(v^*)\sqrt{HSAT \ln \tfrac{2T}{\delta}}.
\end{align*}
\end{lemma}
This lemma is one of our key technical contributions.
To prove this lemma one can write
\begin{align*}
    &\sum_{t=1}^T \left(V^*(s_t)- Q^*(s_t,a_t)\right) \\
    &= \sum_{t=1}^T (V^*(s_t) - \hat{V}_t(s_t))+ \sum_{t=1}^T (\hat{Q}_t(s_t, a_t) - Q^*(s_t, a_t)),
\end{align*}
using the fact that $\hat{V}_t(s_t) = \hat{Q}_t(s_t, a_t)$ by the greedy policy.
The main part of the proof is to show that the second summation can in fact be bounded as $\sum_{t=2}^{T+1} (\hat{V}_{t}(s_{t}) - V^*(s_{t}))$ plus a small sub-linear term, which cancels with the first summation.

\begin{lemma}
\label{lem: discounted bellman}
With probability at least $1 - \delta$,
\begin{align*}
&\sum_{t=1}^T \left(Q^*(s_t, a_t) - \gamma V^*(s_t) - r(s_t, a_t)\right) \\
& \leq 2 \spn(v^*) \sqrt{2T \ln \tfrac{1}{\delta}} + 2 \spn(v^*).
\end{align*}
\end{lemma}
This lemma is proven via Bellman equation for the discounted setting and Azuma's inequality.

\section{$\tilde{\mathcal{O}}(\sqrt{T})$ Regret for Ergodic MDPs}
\label{section: policy optimization}
In this section, we propose another model-free algorithm that achieves $\tilde{\mathcal{O}}(\sqrt{T})$ regret bound for \textit{ergodic} MDPs, a sub-class of weakly communicating MDPs. An MDP is ergodic if for any stationary policy $\pi$, the induced Markov chain is \textit{irreducible} and \textit{aperiodic}. Learning ergodic MDPs is arguably easier than the general case because the MDP is explorative by itself.  However, achieving $\tilde{\mathcal{O}}(\sqrt{T})$ regret bound in this case with model-free methods is still highly non-trivial and we are not aware of any such result in the literature. 
Below, we first introduce a few useful properties of ergodic MDPs,
all of which can be found in~\citep{puterman2014markov}.

We use \textit{randomized policies} in this approach. A randomized policy $\pi$ maps every state $s$ to a distribution over actions $\pi(\cdot|s)\in\Delta_A$, where $\Delta_A=\{x\in \mathbb{R}_+^A: \sum_{a} x(a)=1\}$. In an ergodic MDP, any policy $\pi$ induces a Markov chain with a unique stationary distribution $\mu^\pi\in\Delta_S$ satisfying $(\mu^\pi)^\top P^\pi = (\mu^\pi)^\top$, where $P^{\pi}\in \mathbb{R}^{S\times S}$ is the induced transition matrix defined as $P^\pi(s, s') = \sum_a \pi(a|s)p(s'|s,a)$. 
We denote the stationary distribution of the optimal policy $\pi^*$ by $\mu^*$.

For ergodic MDPs, the long-term average reward $J^\pi$ of any fixed policy $\pi$ is independent of the starting state and can be written as $J^\pi=(\mu^\pi)^\top r^\pi$ where $r^\pi \in [0,1]^S$ is such that $r^\pi(s):=\sum_a \pi(a|s)r(s,a)$. For any policy $\pi$, the following Bellman equation has a solution $q^\pi: \calS\times \calA\rightarrow \mathbb{R}$ that is unique up to an additive constant: 
\begin{align*}
     J^\pi + q^\pi(s,a) = r(s,a) + \E_{s'\sim p(\cdot|s,a)}[v^\pi(s')],  
\end{align*}
where $v^\pi(s)=\sum_a \pi(a|s)q^\pi(s,a)$. In this section, we impose an extra constraint: $\sum_s \mu^{\pi}(s) v^\pi(s)=0$ so that $q^\pi$ is indeed unique. In this case, it can be shown that $v^\pi$ has the following form: 
\begin{align}
     v^\pi(s) = \sum_{t=0}^\infty \left(\mathbf{e}_s^\top (P^\pi)^t - (\mu^\pi)^\top \right) r^\pi \label{eqn:v_pi_identity}
\end{align}
where $\mathbf{e}_s$ is the basis vector with $1$ in coordinate $s$. 

Furthermore, ergodic MDPs have finite \textit{mixing time} and \textit{hitting time}, defined as follows. 
\begin{definition}[\citep{levin2017markov, wang2017primal}]
\label{def:mixing}
The mixing time of an ergodic MDP is defined as 
\begin{align*}
     \mix:=\max_{\pi}\min\left\{ t\geq 1~\Big\vert ~ \| (P^{\pi})^t(s,\cdot) - \mu^\pi \|_1\leq \frac{1}{4}, \forall s \right\},  
\end{align*}
that is, the maximum time required for any policy starting at any initial state to make the state distribution $\frac{1}{4}$-close (in $\ell_1$ norm) to the stationary distribution. 
\end{definition}
\begin{definition}
\label{def:hitting}
The hitting time of an ergodic MDP is defined as 
\[
\hit := \max_{\pi}\max_s \frac{1}{\mu^\pi(s)},
\]  
that is, the maximum inverse stationary probability of visiting any state under any policy.
\end{definition}
Our regret bound also depends on the following {\it distribution mismatch coefficient}:
\begin{align}
        \ratio &:= \max_{\pi} \sum_s \frac{\mu^*(s)}{\mu^\pi(s)} \label{eqn:ratio}
\end{align} 
which has been used in previous work~\citep{kakade2002approximately, agarwal2019optimality}. 
Clearly, one has $\ratio \leq \hit\sum_s\mu^*(s) = \hit$.
Note that these quantities are all parameters of the MDP only and are considered as finite constants compared to the horizon $T$.
We thus assume that $T$ is large enough so that $\mix$ and $\hit$ are both smaller than $T/4$. 
Also, we assume that these quantities are known to the algorithm.

\subsection{Policy Optimization via Optimistic OMD}
\label{sec:PO}

The key to get $\otil(\sqrt{T})$ bound is to learn the optimal policy $\pi^*$ directly, by reducing the problem to solving an adversarial multi-armed bandit (MAB)~\citep{auer2002nonstochastic} instance at each individual state.


\setcounter{AlgoLine}{0}
\begin{algorithm}[t]
\caption{\pg}
\label{alg: PG-online-2}
\textbf{Define:} episode length $B=16\mix\hit(\log_2 T)^2$ and number of episodes $K=T/B$ \\
\textbf{Initialize:} $\pi_1'(a|s)=\pi_1(a|s)=\frac{1}{A}, \forall s, a$. \\
\For{$k=1, 2, \ldots, K$}{
    \For{$t=(k-1)B+1, \ldots, kB$}{
        Draw $a_t\sim \pi_k(\cdot|s_t)$ and observe $s_{t+1}$. 
    }
    Define trajectory $\calT_k = $ \\ \hspace*{20pt} $(s_{(k-1)B+1}, a_{(k-1)B+1}, \ldots, s_{kB}, a_{kB})$. \\
    \For{all $s\in\calS$}{
          $\widehat{\beta}_k(s,\cdot) = \estimate(\calT_k, \pi_k, s)$. \\
          $(\pi_{k+1}'(\cdot|s), \pi_{k+1}(\cdot|s)) = $\\
          \hspace*{20pt} $ \oomd(\pi_k'(\cdot|s), \widehat{\beta}_k(s,\cdot))$.
    }
}
\end{algorithm}
\setcounter{AlgoLine}{0}
\begin{algorithm}[h!]
\caption{\estimate}
\label{alg: estimate}
\textbf{Input: } $\calT, \pi, s$
\begin{align*}
 \calT&: \text{\ a state-action trajectory from $t_1$ to $t_2$} \\
        &\quad (s_{t_1}, a_{t_1}, \ldots, s_{t_2}, a_{t_2}) \\
\pi&: \text{a policy used to sample the trajectory $\calT$} \\
 s&: \text{\ target state} 
\end{align*}\\
\textbf{Define:} $N =4\mix\log_2 T$ ({\small window length minus $1$}) \\
\textbf{Initialize:} $\tau\leftarrow t_1$, $i\leftarrow 0$\\
\nl \While{$\tau\leq t_2-N$}{
\nl      \If{$s_\tau=s$}{
\nl	        $i\leftarrow i+1$ \\
\nl             Let $R=\sum_{t=\tau}^{\tau+N}r(s_t,a_t)$. \\      
\nl             Let \scalebox{1}{$y_{i}(a) = \frac{R}{\pi(a|s)}\one[a_\tau=a],\forall a$}. \ \  ($y_i\in \mathbb{R}^A$)\label{line:IPS}\\
\nl             $\tau \leftarrow \tau+2N$   \label{line: skip 2N}  
      }
\nl      \Else{
             $\tau \leftarrow \tau+1$
      }
}
\nl \If{$i\neq 0$}{
      \textbf{return\ } $\frac{1}{i}\sum_{j=1}^i y_j$.  
}
\nl \Else{
     \textbf{return\ } $\mathbf{0}$. 
}
\end{algorithm}

\setcounter{AlgoLine}{0}
\begin{algorithm}[h!]
\caption{\oomd}
\label{alg:oomd}
\textbf{Input:} $\pi'\in \Delta_A, \widehat{\beta}\in \mathbb{R}^A$ \\
\textbf{Define:} \\
\quad Regularizer $\psi(x) = \frac{1}{\eta}\sum_{a=1}^A \log \frac{1}{x(a)}, \text{\ for\ }x\in \mathbb{R}_{+}^A$\\
\quad Bregman divergence associated with $\psi$: \[D_\psi(x, x') = \psi(x)-\psi(x') - \langle \nabla \psi(x'), x-x' \rangle \] 
\ \\
\textbf{Update:}
\begin{align}
      \pi_{next}' &= \argmax_{\pi\in\Delta_A}\left\{ \langle  \pi,  \widehat{\beta} \rangle - D_{\psi}(\pi, \pi')  \right\} \label{eqn:OMD} \\
      \pi_{next} &= \argmax_{\pi\in\Delta_A}\left\{ \langle  \pi,   \widehat{\beta} \rangle - D_{\psi}(\pi, \pi_{next}')  \right\} \label{eqn:optimism}
\end{align}\\
\textbf{return\ } $(\pi_{next}', \pi_{next})$. 
\end{algorithm}

The details of our algorithm \pg is shown in Algorithm~\ref{alg: PG-online-2}.
It proceeds in episodes, and maintains an independent copy of a specific MAB algorithm for each  state. At the beginning of episode $k$, each MAB algorithm outputs an action distribution $\pi_k(\cdot|s)$ for the corresponding state $s$, which together induces a policy $\pi_k$. The learner then executes policy $\pi_k$ throughout episode $k$. At the end of the episode, for every state $s$ we feed a reward estimator $\widehat{\beta}_k(s,\cdot) \in \mathbb{R}^A$ to the corresponding MAB algorithm, where $\widehat{\beta}_k$ is constructed using the samples collected in episode $k$ (see Algorithm~\ref{alg: estimate}). Finally all MAB algorithms update their distributions and output $\pi_{k+1}$ for the next episode (Algorithm~\ref{alg:oomd}). 

The reward estimator $\widehat{\beta}_k(s,\cdot)$ is an almost unbiased estimator for 
\begin{equation}\label{eqn:beta}
\beta^{\pi_k}(s,\cdot):=q^{\pi_k}(s,\cdot) +NJ^{\pi_k}
\end{equation}
with negligible bias ($N$ is defined in Algorithm~\ref{alg: estimate}).  
The term $NJ^{\pi_k}$ is the same for all actions and thus the corresponding MAB algorithm is trying to learn the best action at state $s$ in terms of the average of Q-value functions $q^{\pi_1}(s,\cdot), \ldots, q^{\pi_K}(s,\cdot)$.
To construct the reward estimator for state $s$, the sub-routine \estimate collects non-overlapping intervals of length $N+1=\otil(\mix)$ that start from state $s$, and use the standard inverse-propensity scoring to construct an estimator $y_i$ for interval $i$ (Line~\ref{line:IPS}). In fact, to reduce the correlation among the non-overlapping intervals, we also make sure that these intervals are at least $N$ steps apart from each other (Line~\ref{line: skip 2N}). The final estimator $\widehat{\beta}_k(s,\cdot)$ is simply the average of all estimators $y_i$ over these disjoint intervals. 
This averaging is important for reducing variance as explained later (see also Lemma~\ref{lemma: PG unbiasedness}).

The MAB algorithm we use is \textit{optimistic online mirror descent} (OOMD)~\citep{rakhlin2013online} with \textit{log-barrier} as the regularizer, analyzed in depth in \citep{wei2018more}.
Here, optimism refers to something different from the optimistic exploration in Section~\ref{sec: qlearning}.
It corresponds to the fact that after a standard mirror descent update (Eq.~\eqref{eqn:OMD}), the algorithm further makes a similar update using an optimistic prediction of the next reward vector, which in our case is simply the previous reward estimator (Eq.~\eqref{eqn:optimism}).
We refer the reader to \citep{wei2018more} for more details,
but point out that the optimistic prediction we use here is new.

It is clear that each MAB algorithm faces a non-stochastic problem (since $\pi_k$ is changing over time) and thus it is important to deploy an adversarial MAB algorithm.
The standard algorithm for adversarial MAB is \expthree~\citep{auer2002nonstochastic},
which was also used for solving adversarial MDPs~\citep{neu2013online} (more comparisons with this to follow).
However, there are several important reasons for our choice of the recently developed OOMD with log-barrier:
\begin{itemize}
\item
First, the log-barrier regularizer produces a more exploratory distribution compared to \expthree (as noticed in e.g.~\citep{agarwal2017corralling}), so we do not need an explicit exploration over the actions, which significantly simplifies the analysis compared to \citep{neu2013online}. 

\item
Second, log-barrier regularizer provides more \textit{stable} updates compared to \expthree in the sense that $\pi_k(a|s)$ and $\pi_{k-1}(a|s)$ are within a multiplicative factor of each other (see Lemma~\ref{lemma: stability lemma}). This implies that the corresponding policies and their Q-value functions are also stable, which is critical for our analysis. 

\item
Finally, the optimistic prediction of OOMD, together with our particular reward estimator from \estimate, provides a variance reduction effect that leads to a better regret bound in terms of $\ratio$ instead of $\hit$.
See Lemma~\ref{lemma: PG reduction to bandit} and Lemma~\ref{lemma: PG regret main term}.
\end{itemize}

The regret guarantee of our algorithm is shown below. 

\begin{theorem}
     \label{thm: policy optimization thm}
     For ergodic MDPs, with an appropriate chosen learning rate $\eta$ for Algorithm~\ref{alg:oomd}, MDP-OOMD achieves 
     \begin{align*}
          \E[R_T] = \otil\left(\sqrt{\mix^3\ratio AT}  \right). 
     \end{align*}
\end{theorem}

Note that in this bound, the dependence on the number of states $S$ is hidden in $\rho$, since $\rho \geq  \sum_s \frac{\mu^*(s)}{\mu^*(s)} = S$.
Compared to the bound of Algorithm~\ref{alg: optimistic qlearning} or some other model-based algorithms such as UCRL2, this bound has an extra dependence on $\mix$, a potentially large constant.
As far as we know, all existing mirror-descent-based algorithms for the average-reward setting has the same issue (such as~\citep{neu2013online, wang2017primal, abbasi2019politex}).
The role of $\mix$ in our analysis is almost the same as that of $1/(1-\gamma)$ in the discounted setting ($\gamma$ is the discount factor). Specifically, a small $\mix$ ensures 1) a short trajectory needed to approximate the Q-function with expected trajectory reward (in view of Eq.~\eqref{eqn: unbiased mean}) and 2) an upper bound for the magnitude of $q(s,a)$ and $v(s)$ (Lemma~\ref{lemma: bounded span of pi}). 
For the discounted setting these are ensured by the discount factor already.

\paragraph{Comparisons.} 
\citet{neu2013online} considered learning ergodic MDPs with {\it known} transition kernel and {\it adversarial} rewards,  a setting incomparable to ours.
Their algorithm maintains a copy of \expthree for each state, but the reward estimators fed to these algorithms are constructed using the knowledge of the transition kernel and are very different from ours.
They proved a regret bound of order $\otil\left(\sqrt{\mix^3\hit AT}\right)$,
which is worse than ours since $\ratio\leq\hit$.

In another recent work, \cite{wang2017primal} considered learning ergodic MDPs under the assumption that the learner is provided with a generative model (an oracle that takes in a state-action pair and output a sample of the next state). They derived a sample-complexity bound of order $\otil\left(\frac{\mix^2\tau^2 SA}{\epsilon^2}\right)$ for finding an $\epsilon$-optimal policy, where $\tau=\max\left\{ \max_s \left(\frac{\mu^*(s)}{1/S}\right)^2, \max_{s', \pi}\left(\frac{1/S}{\mu^{\pi}(s')}\right)^2\right\} $, which is at least $\max_{\pi}\max_{s,s'}\frac{\mu^*(s)}{\mu^{\pi}(s')}$ by AM-GM inequality. 
This result is again incomparable to ours, but we point out that our distribution mismatch coefficient $\ratio$ is always bounded by $\tau S$, while $\tau$ can be much larger than $\ratio$ on the other hand.

Finally, \citet{abbasi2019politex} considers a more general setting with function approximation, and their algorithm \politex maintains a copy of the standard exponential weight algorithm for each state, very similar to~\citep{neu2013online}. 
When specified to our tabular setting, one can verify (according to their Theorem 5.2) that \politex achieves $\mix^3\hit\sqrt{SA}T^\frac{3}{4}$ regret, which is significantly worse than ours in terms of all parameters.

\subsection{Proof sketch of Theorem~\ref{thm: policy optimization thm}}
\label{subsection: pg analysis}
We first decompose the regret as follows: 
\begin{align}
     &R_T = \sum_{t=1}^T J^* - r(s_t,a_t) \nonumber \\
     &= B\sum_{k=1}^K \left(J^* - J^{\pi_k}\right) + \sum_{k=1}^K \sum_{t\in\calI_k} \left( J^{\pi_k} - r(s_t,a_t)\right), \label{eqn: PG regret decompose}
\end{align}
where $\calI_k:= \{(k-1)B+1, \ldots, kB\}$ is the set of time steps for episode $k$. 
Using the \textit{reward difference lemma} (Lemma~\ref{lemma: reward diff lemma} in the appendix), the first term of Eq.~\eqref{eqn: PG regret decompose} can be written as
\begin{align*}
    &B\sum_s \mu^*(s) \left[ \sum_{k=1}^K \sum_a(\pi^*(a|s)-\pi_k(a|s))q^{\pi_k}(s,a)\right], 
\end{align*}
where the term in the square bracket can be recognized as exactly the regret of the MAB algorithm for state $s$ and is analyzed in Lemma~\ref{lemma: PG reduction to bandit} of Section~\ref{sec:lemmas}.
Combining the regret of all MAB algorithms,
Lemma~\ref{lemma: PG regret main term} then shows that in expectation the first term of Eq.~\eqref{eqn: PG regret decompose} is at most
\begin{align}
     \otil\left( \frac{BA}{\eta} + \frac{\eta T N^3 \ratio}{B} + \eta^3 TN^6  \right).  \label{eqn: PG first term}
\end{align}

On the other hand, the expectation of the second term in Eq.\eqref{eqn: PG regret decompose} can be further written as
\begin{align}
    &\E\left[\sum_{k=1}^K\sum_{t\in\calI_k} (J^{\pi_k} - r(s_t,a_t))\right] \nonumber \\
    &=\E\left[\sum_{k=1}^K \sum_{t\in\calI_k} (\E_{s'\sim p(\cdot|s_t,a_t)} [v^{\pi_k}(s')] - q^{\pi_k}(s_t, a_t))\right] \tag{Bellman equation} \\
    &= \E\left[\sum_{k=1}^K\sum_{t\in\calI_k} (\E_{s'\sim p(\cdot|s_t,a_t)} [v^{\pi_k}(s')] - v^{\pi_k}(s_{t+1})) \right] \nonumber \\
    &\qquad \quad + \E\left[\sum_{k=1}^K \sum_{t\in\calI_k} (v^{\pi_k}(s_{t}) - q^{\pi_k}(s_t, a_t))\right] \nonumber \\
    &\qquad \quad + \E\left[\sum_{k=1}^K \sum_{t\in\calI_k} (v^{\pi_k}(s_{t+1}) - v^{\pi_k}(s_t))\right]\nonumber \\
    &=\E\left[\sum_{k=1}^K (v^{\pi_k}(s_{kB+1}) - v^{\pi_k}(s_{(k-1)B+1}))\right] \tag{the first two terms above are zero}\nonumber \\
    &= \E\left[\sum_{k=1}^{K-1} (v^{\pi_k}(s_{kB+1}) - v^{\pi_{k+1}}(s_{kB+1}))\right] \nonumber  \\ 
    &\qquad \quad + \E\Big[v^{\pi_K}(s_{KB+1})-v^{\pi_1}(s_1)\Big]. 
\end{align}
The first term in the last expression can be bounded by $\order(\eta N^3K) = \order(\eta N^3T/B)$ due to the stability of \oomd (Lemma~\ref{lemma: stability lemma}) and the second term is at most $\order(\mix)$ according to Lemma~\ref{lemma: bounded span of pi} in the appendix.  

Combining these facts with $N = \otil(\mix)$, $B=\otil(\mix\hit)$, Eq.~\eqref{eqn: PG regret decompose} and Eq.~\eqref{eqn:  PG first term} and choosing the optimal $\eta$, we arrive at
\begin{align*}
    &\E[R_T]  = \otil\left(\frac{BA}{\eta} + \eta \frac{\mix^3 \ratio T}{B} + \eta^3 \mix^6T\right)\\
    &= \otil\left(\sqrt{\mix^3 \ratio AT} + \left(\mix^3\hit A\right)^{\frac{3}{4}}T^{\frac{1}{4}} + \mix^2\hit A\right). 
\end{align*}

\subsection{Auxiliary Lemmas}\label{sec:lemmas}
To analyze the regret, we establish several useful lemmas, whose proofs can be found in the Appendix. First, we show that $\widehat{\beta}_k(s,a)$ is an almost unbiased estimator for $\beta^{\pi_k}(s,a)$. 
\begin{lemma}
    \label{lemma: PG unbiasedness}
      Let $\E_{k}[x]$ denote the expectation of a random variable $x$ conditioned on all history before episode $k$. Then for any $k, s,a$ (recall $\beta$ defined in Eq.~\eqref{eqn:beta}), 
\begin{align}
    &\Big\lvert \E_k \left[ \widehat{\beta}_k(s,a) \right] - \beta^{\pi_k}(s,a)\Big\rvert \leq \order\left(\frac{1}{T}\right),  \label{eqn: unbiased mean}  \\
    &\E_k \left[\left( \widehat{\beta}_k(s,a) - \beta^{\pi_k}(s,a)\right)^2 \right] \leq \order\left(\frac{N^3 \log T}{B\pi_k(a|s)\mu^{\pi_k}(s)}\right). \label{eqn: unbiased variance}
\end{align}
\end{lemma}

The next lemma shows that in OOMD, $\pi_k$ and $\pi_{k-1}$ are close in a strong sense,
which further implies the stability for several other related quantities.
\begin{lemma}
      \label{lemma: stability lemma}
      For any $k, s, a$, 
      \begin{align}
           \left\lvert \pi_k(a|s) - \pi_{k-1}(a|s) \right\rvert &\leq \order(\eta N \pi_{k-1}(a|s)),  \label{eqn: OMD eta stable}  \\
           \left\vert J^{\pi_k}-J^{\pi_{k-1}} \right\rvert  &\leq \order(\eta N^2),  \nonumber  \\
           \left\lvert v^{\pi_k}(s)-v^{\pi_{k-1}}(s)  \right\rvert &\leq \order(\eta N^3),   \nonumber \\
           \left\lvert q^{\pi_k}(s,a)-q^{\pi_{k-1}}(s,a)  \right\rvert &\leq \order(\eta N^3), \nonumber  \\
           \left\lvert \beta^{\pi_k}(s,a)-\beta^{\pi_{k-1}}(s,a)  \right\rvert &\leq \order(\eta N^3).  \nonumber 
      \end{align}
\end{lemma}

The next lemma shows the regret bound of OOMD based on an analysis similar to~\citep{wei2018more}.
\begin{lemma}
    \label{lemma: PG reduction to bandit}
    For a specific state $s$, we have
\begin{align*}
    &\E\left[\sum_{k=1}^K \sum_a  (\pi^*(a|s)-\pi_k(a|s))\widehat{\beta}_k(s,a)\right] \leq  \order\Bigg(\frac{A\ln T}{\eta}\\
    & + \eta \E\left[\sum_{k=1}^K \sum_a \pi_k(a|s)^2 \left( \widehat{\beta}_k(s,a)- \widehat{\beta}_{k-1}(s,a) \right)^2\right]\Bigg),
\end{align*}
where we define $\widehat{\beta}_0(s,a) =0$ for all $s$ and $a$.
\end{lemma}

Finally, we state a key lemma for proving Theorem~\ref{thm: policy optimization thm}.  
\begin{lemma}
     \label{lemma: PG regret main term}
     \pg ensures
     \begin{align}
     &\E\left[B\sum_{k=1}^K\sum_s\sum_a \mu^*(s)\left( \pi^*(a|s)-\pi_k(a|s) \right) q^{\pi_k}(s,a) \right] \nonumber \\
     &=  \order\left(\frac{BA\ln T}{\eta} + \eta \frac{TN^3\ratio}{B} + \eta^3 TN^6\right).  \nonumber 
     \end{align}
\end{lemma}

\section{Conclusions}
In this work we propose two model-free algorithms for learning infinite-horizon average-reward MDPs.
They are based on different ideas: one reduces the problem to the discounted version, while the other optimizes the policy directly via a novel application of adaptive adversarial multi-armed bandit algorithms.
The main open question is how to achieve the information-theoretically optimal regret bound via a model-free algorithm, if it is possible at all.
We believe that the techniques we develop in this work would be useful in answering this question.

\section*{Acknowledgements}
The authors would like to thank Csaba Szepesvari for pointing out the related works \citep{abbasi2019politex, abbasi2019exploration}, Mengxiao Zhang for helping us prove Lemma~\ref{lemma: PG unbiasedness}, Gergely Neu for clarifying the analysis in \citep{neu2013online}, and Ronan Fruit for discussions on a related open problem presented at ALT 2019.  
Support from NSF for MJ (award ECCS-1810447), HL (award IIS-1755781), HS (award CCF-1817212) and RJ (awards ECCS-1810447 and CCF-1817212) is gratefully acknowledged.

\bibliography{online_rl}

\begin{thebibliography}{40}
\providecommand{\natexlab}[1]{#1}
\providecommand{\url}[1]{\texttt{#1}}
\expandafter\ifx\csname urlstyle\endcsname\relax
  \providecommand{\doi}[1]{doi: #1}\else
  \providecommand{\doi}{doi: \begingroup \urlstyle{rm}\Url}\fi

\bibitem[Abbasi-Yadkori et~al.(2019{\natexlab{a}})Abbasi-Yadkori, Bartlett,
  Bhatia, Lazic, Szepesvari, and Weisz]{abbasi2019politex}
Abbasi-Yadkori, Y., Bartlett, P., Bhatia, K., Lazic, N., Szepesvari, C., and
  Weisz, G.
\newblock Politex: Regret bounds for policy iteration using expert prediction.
\newblock In \emph{International Conference on Machine Learning}, pp.\
  3692--3702, 2019{\natexlab{a}}.

\bibitem[Abbasi-Yadkori et~al.(2019{\natexlab{b}})Abbasi-Yadkori, Lazic,
  Szepesvari, and Weisz]{abbasi2019exploration}
Abbasi-Yadkori, Y., Lazic, N., Szepesvari, C., and Weisz, G.
\newblock Exploration-enhanced politex.
\newblock \emph{arXiv preprint arXiv:1908.10479}, 2019{\natexlab{b}}.

\bibitem[Agarwal et~al.(2017)Agarwal, Luo, Neyshabur, and
  Schapire]{agarwal2017corralling}
Agarwal, A., Luo, H., Neyshabur, B., and Schapire, R.~E.
\newblock Corralling a band of bandit algorithms.
\newblock In \emph{Conference on Learning Theory}, pp.\  12--38, 2017.

\bibitem[Agarwal et~al.(2019)Agarwal, Kakade, Lee, and
  Mahajan]{agarwal2019optimality}
Agarwal, A., Kakade, S.~M., Lee, J.~D., and Mahajan, G.
\newblock Optimality and approximation with policy gradient methods in markov
  decision processes.
\newblock \emph{arXiv preprint arXiv:1908.00261}, 2019.

\bibitem[Agrawal \& Jia(2017)Agrawal and Jia]{agrawal2017optimistic}
Agrawal, S. and Jia, R.
\newblock Optimistic posterior sampling for reinforcement learning: worst-case
  regret bounds.
\newblock In \emph{Advances in Neural Information Processing Systems}, pp.\
  1184--1194, 2017.

\bibitem[Auer \& Ortner(2007)Auer and Ortner]{auer2007logarithmic}
Auer, P. and Ortner, R.
\newblock Logarithmic online regret bounds for undiscounted reinforcement
  learning.
\newblock In \emph{Advances in Neural Information Processing Systems}, pp.\
  49--56, 2007.

\bibitem[Auer et~al.(2002)Auer, Cesa-Bianchi, Freund, and
  Schapire]{auer2002nonstochastic}
Auer, P., Cesa-Bianchi, N., Freund, Y., and Schapire, R.~E.
\newblock The nonstochastic multiarmed bandit problem.
\newblock \emph{SIAM Journal on Computing}, 32\penalty0 (1):\penalty0 48--77,
  2002.

\bibitem[Azar et~al.(2017)Azar, Osband, and Munos]{azar2017minimax}
Azar, M.~G., Osband, I., and Munos, R.
\newblock Minimax regret bounds for reinforcement learning.
\newblock In \emph{Proceedings of the 34th International Conference on Machine
  Learning-Volume 70}, pp.\  263--272. JMLR. org, 2017.

\bibitem[Bartlett \& Tewari(2009)Bartlett and Tewari]{bartlett2009regal}
Bartlett, P.~L. and Tewari, A.
\newblock Regal: A regularization based algorithm for reinforcement learning in
  weakly communicating mdps.
\newblock In \emph{Proceedings of the Twenty-Fifth Conference on Uncertainty in
  Artificial Intelligence}, pp.\  35--42. AUAI Press, 2009.

\bibitem[Bubeck et~al.(2019)Bubeck, Li, Luo, and Wei]{bubeck2019improved}
Bubeck, S., Li, Y., Luo, H., and Wei, C.-Y.
\newblock Improved path-length regret bounds for bandits.
\newblock In \emph{Conference On Learning Theory}, 2019.

\bibitem[Chiang et~al.(2012)Chiang, Yang, Lee, Mahdavi, Lu, Jin, and
  Zhu]{chiang2012online}
Chiang, C.-K., Yang, T., Lee, C.-J., Mahdavi, M., Lu, C.-J., Jin, R., and Zhu,
  S.
\newblock Online optimization with gradual variations.
\newblock In \emph{Conference on Learning Theory}, pp.\  6--1, 2012.

\bibitem[Dong et~al.(2019)Dong, Wang, Chen, and Wang]{dong2019q}
Dong, K., Wang, Y., Chen, X., and Wang, L.
\newblock Q-learning with ucb exploration is sample efficient for
  infinite-horizon mdp.
\newblock \emph{arXiv preprint arXiv:1901.09311}, 2019.

\bibitem[Fruit et~al.(2018{\natexlab{a}})Fruit, Pirotta, and
  Lazaric]{fruit2018near}
Fruit, R., Pirotta, M., and Lazaric, A.
\newblock Near optimal exploration-exploitation in non-communicating markov
  decision processes.
\newblock In \emph{Advances in Neural Information Processing Systems}, pp.\
  2994--3004, 2018{\natexlab{a}}.

\bibitem[Fruit et~al.(2018{\natexlab{b}})Fruit, Pirotta, Lazaric, and
  Ortner]{fruit2018efficient}
Fruit, R., Pirotta, M., Lazaric, A., and Ortner, R.
\newblock Efficient bias-span-constrained exploration-exploitation in
  reinforcement learning.
\newblock In \emph{International Conference on Machine Learning}, pp.\
  1573--1581, 2018{\natexlab{b}}.

\bibitem[Fruit et~al.(2019)Fruit, Pirotta, and Lazaric]{fruit2019improved}
Fruit, R., Pirotta, M., and Lazaric, A.
\newblock Improved analysis of ucrl2b, 2019.
\newblock Available at \\ rlgammazero.github.io/docs/ucrl2b\_improved.pdf.

\bibitem[Gopalan \& Mannor(2015)Gopalan and Mannor]{gopalan2015thompson}
Gopalan, A. and Mannor, S.
\newblock Thompson sampling for learning parameterized markov decision
  processes.
\newblock In \emph{Conference on Learning Theory}, pp.\  861--898, 2015.

\bibitem[Jaksch et~al.(2010)Jaksch, Ortner, and Auer]{jaksch2010near}
Jaksch, T., Ortner, R., and Auer, P.
\newblock Near-optimal regret bounds for reinforcement learning.
\newblock \emph{Journal of Machine Learning Research}, 11\penalty0
  (Apr):\penalty0 1563--1600, 2010.

\bibitem[Jin et~al.(2018)Jin, Allen-Zhu, Bubeck, and Jordan]{jin2018q}
Jin, C., Allen-Zhu, Z., Bubeck, S., and Jordan, M.~I.
\newblock Is {Q}-learning provably efficient?
\newblock In \emph{Advances in Neural Information Processing Systems}, pp.\
  4863--4873, 2018.

\bibitem[Kakade \& Langford(2002)Kakade and Langford]{kakade2002approximately}
Kakade, S. and Langford, J.
\newblock Approximately optimal approximate reinforcement learning.
\newblock In \emph{Proceedings of the 34th International Conference on Machine
  Learning}, 2002.

\bibitem[Lattimore \& Szepesv{\'a}ri(2018)Lattimore and
  Szepesv{\'a}ri]{lattimore2018bandit}
Lattimore, T. and Szepesv{\'a}ri, C.
\newblock Bandit algorithms.
\newblock \emph{Cambridge University Press}, 2018.

\bibitem[Levin \& Peres(2017)Levin and Peres]{levin2017markov}
Levin, D.~A. and Peres, Y.
\newblock \emph{Markov chains and mixing times}, volume 107.
\newblock American Mathematical Soc., 2017.

\bibitem[Mnih et~al.(2013)Mnih, Kavukcuoglu, Silver, Graves, Antonoglou,
  Wierstra, and Riedmiller]{mnih2013playing}
Mnih, V., Kavukcuoglu, K., Silver, D., Graves, A., Antonoglou, I., Wierstra,
  D., and Riedmiller, M.
\newblock Playing atari with deep reinforcement learning.
\newblock \emph{arXiv preprint arXiv:1312.5602}, 2013.

\bibitem[Mnih et~al.(2016)Mnih, Badia, Mirza, Graves, Lillicrap, Harley,
  Silver, and Kavukcuoglu]{mnih2016asynchronous}
Mnih, V., Badia, A.~P., Mirza, M., Graves, A., Lillicrap, T., Harley, T.,
  Silver, D., and Kavukcuoglu, K.
\newblock Asynchronous methods for deep reinforcement learning.
\newblock In \emph{International conference on machine learning}, pp.\
  1928--1937, 2016.

\bibitem[Neu et~al.(2013)Neu, Gy{\"o}rgy, Szepesv{\'a}ri, and
  Antos]{neu2013online}
Neu, G., Gy{\"o}rgy, A., Szepesv{\'a}ri, C., and Antos, A.
\newblock Online markov decision processes under bandit feedback.
\newblock \emph{IEEE Transactions on Automatic Control}, 59:\penalty0 676--691,
  2013.

\bibitem[Ortner(2018)]{ortner2018regret}
Ortner, R.
\newblock Regret bounds for reinforcement learning via markov chain
  concentration.
\newblock \emph{arXiv preprint arXiv:1808.01813}, 2018.

\bibitem[Osband et~al.(2014)Osband, Van~Roy, and Wen]{osband2014generalization}
Osband, I., Van~Roy, B., and Wen, Z.
\newblock Generalization and exploration via randomized value functions.
\newblock \emph{arXiv preprint arXiv:1402.0635}, 2014.

\bibitem[Ouyang et~al.(2017{\natexlab{a}})Ouyang, Gagrani, and
  Jain]{ouyang2017learningbased}
Ouyang, Y., Gagrani, M., and Jain, R.
\newblock Learning-based control of unknown linear systems with thompson
  sampling.
\newblock \emph{arXiv preprint arXiv:1709.04047}, 2017{\natexlab{a}}.

\bibitem[Ouyang et~al.(2017{\natexlab{b}})Ouyang, Gagrani, Nayyar, and
  Jain]{ouyang2017learning}
Ouyang, Y., Gagrani, M., Nayyar, A., and Jain, R.
\newblock Learning unknown markov decision processes: A thompson sampling
  approach.
\newblock In \emph{Advances in Neural Information Processing Systems}, pp.\
  1333--1342, 2017{\natexlab{b}}.

\bibitem[Puterman(2014)]{puterman2014markov}
Puterman, M.~L.
\newblock \emph{Markov decision processes: discrete stochastic dynamic
  programming}.
\newblock John Wiley \& Sons, 2014.

\bibitem[Rakhlin \& Sridharan(2013)Rakhlin and Sridharan]{rakhlin2013online}
Rakhlin, A. and Sridharan, K.
\newblock Online learning with predictable sequences.
\newblock In \emph{Conference on Learning Theory}, pp.\  993--1019, 2013.

\bibitem[Schulman et~al.(2015)Schulman, Levine, Abbeel, Jordan, and
  Moritz]{schulman2015trust}
Schulman, J., Levine, S., Abbeel, P., Jordan, M., and Moritz, P.
\newblock Trust region policy optimization.
\newblock In \emph{International conference on machine learning}, pp.\
  1889--1897, 2015.

\bibitem[Schwartz(1993)]{schwartz1993reinforcement}
Schwartz, A.
\newblock A reinforcement learning method for maximizing undiscounted rewards.
\newblock In \emph{Proceedings of the tenth international conference on machine
  learning}, volume 298, pp.\  298--305, 1993.

\bibitem[Strehl \& Littman(2008)Strehl and Littman]{strehl2008analysis}
Strehl, A.~L. and Littman, M.~L.
\newblock An analysis of model-based interval estimation for markov decision
  processes.
\newblock \emph{Journal of Computer and System Sciences}, 74\penalty0
  (8):\penalty0 1309--1331, 2008.

\bibitem[Strehl et~al.(2006)Strehl, Li, Wiewiora, Langford, and
  Littman]{strehl2006pac}
Strehl, A.~L., Li, L., Wiewiora, E., Langford, J., and Littman, M.~L.
\newblock Pac model-free reinforcement learning.
\newblock In \emph{Proceedings of the 23rd international conference on Machine
  learning}, pp.\  881--888. ACM, 2006.

\bibitem[Talebi \& Maillard(2018)Talebi and Maillard]{talebi2018variance}
Talebi, M.~S. and Maillard, O.-A.
\newblock Variance-aware regret bounds for undiscounted reinforcement learning
  in mdps.
\newblock In \emph{Algorithmic Learning Theory}, pp.\  770--805, 2018.

\bibitem[Wang(2017)]{wang2017primal}
Wang, M.
\newblock Primal-dual $\pi $ learning: Sample complexity and sublinear run time
  for ergodic markov decision problems.
\newblock \emph{arXiv preprint arXiv:1710.06100}, 2017.

\bibitem[Watkins(1989)]{watkins1989learning}
Watkins, C. J. C.~H.
\newblock \emph{Learning from delayed rewards}.
\newblock Phd Thesis, King's College, Cambridge, 1989.

\bibitem[Wei \& Luo(2018)Wei and Luo]{wei2018more}
Wei, C.-Y. and Luo, H.
\newblock More adaptive algorithms for adversarial bandits.
\newblock In \emph{Conference On Learning Theory}, pp.\  1263--1291, 2018.

\bibitem[Zanette \& Brunskill(2019)Zanette and Brunskill]{zanette2019tighter}
Zanette, A. and Brunskill, E.
\newblock Tighter problem-dependent regret bounds in reinforcement learning
  without domain knowledge using value function bounds.
\newblock In \emph{International Conference on Machine Learning}, 2019.

\bibitem[Zhang \& Ji(2019)Zhang and Ji]{zhang2019regret}
Zhang, Z. and Ji, X.
\newblock Regret minimization for reinforcement learning by evaluating the
  optimal bias function.
\newblock In \emph{Advances in Neural Information Processing Systems}, 2019.

\end{thebibliography}
\bibliographystyle{icml2020}


\newpage
\onecolumn 
\appendix


\section{Omitted Proofs in Section \ref{sec: qlearning}}
In this section, we provide detailed proof for the lemmas used in Section \ref{sec: qlearning}. Recall that the learning rate $\alpha_\tau = \frac{H+1}{H+\tau}$ is similar to the one used by \cite{jin2018q}. For notational convenience, let
\begin{align}
\label{eq: alpha-tau-i}
\alpha_\tau^0 := \prod_{j=1}^\tau (1 - \alpha_j), \qquad \alpha_\tau^i := \alpha_i \prod_{j=i+1}^\tau (1 - \alpha_j). 
\end{align}
It can be verified that $\alpha_\tau^0 = 0$ for $\tau \geq 1$ and we define $\alpha_0^0 = 1$. These quantities are used in the proof of Lemma \ref{lem: v-star minus q-star} and have some nice properties summarized in the following lemma.
\begin{lemma}[\cite{jin2018q}]
\label{lem: learning rate}
The following properties hold for $\alpha_\tau^i$:
\begin{enumerate}
\item $\frac{1}{\sqrt{\tau}} \leq \sum_{i=1}^\tau \frac{\alpha_\tau^i}{\sqrt i} \leq \frac{2}{\sqrt \tau}$ for every $\tau \geq 1$.
\item $\sum_{i=1}^\tau  (\alpha_\tau^i)^2 \leq \frac{2H}{\tau}$ for every $\tau \geq 1$.
\item $\sum_{i=1}^\tau \alpha_\tau^i = 1$ for every $\tau \geq 1$ and $\sum_{\tau=i}^\infty \alpha_\tau^i = 1 + \frac{1}{H}$ for every $i \geq 1$.
\end{enumerate}
\end{lemma}

Also recall the well-known Azuma's inequality:
\begin{lemma}[Azuma's inequality]
\label{lem: azuma}
Let $X_1, X_2, \cdots$ be a martingale difference sequence with $|X_i| \leq c_i$ for all $i$. Then, for any $0 < \delta < 1$, 
\begin{align*}
\mathbb{P}\left(\sum_{i=1}^T X_i \geq \sqrt{2 \bar{c}_T^2 \ln \frac{1}{\delta}}\right) \leq \delta,
\end{align*}
where $\bar{c}_T^2 := \sum_{i=1}^Tc_i^2$.
\end{lemma}

\subsection{Proof of Lemma \ref{lem: discounted average connection}}
\textbf{Lemma \ref{lem: discounted average connection} (Restated).} Let $V^*$ be the optimal value function in the discounted MDP with discount factor $\gamma$ and $v^*$ be the optimal value function in the undiscounted MDP. Then,
\begin{enumerate}
    \item $|J^* - (1-\gamma) V^*(s)| \leq (1-\gamma)\spn(v^*)$, $\forall s\in\calS$,
    \item $\spn(V^*) \leq 2\spn(v^*)$.
\end{enumerate}
\begin{proof}
\begin{enumerate}
    \item Let $\pi^*$ and $\pi_\gamma$ be the optimal policy under undiscounted and discounted settings, respectively. By Bellman's equation, we have
    \begin{align*}
        v^*(s) = r(s, \pi^*(s))-J^* + \E_{s'\sim p(\cdot|s, \pi^*(s))}v^*(s'). 
    \end{align*}
    Consider a state sequence $s_1, s_2, \cdots$ generated by $\pi^*$. Then, by sub-optimality of $\pi^*$ for the discounted setting, we have
    \begin{align*}
        V^*(s_1) &\geq \E\left[\sum_{t=1}^\infty \gamma^{t-1} r(s_t, \pi^*(s_t))  ~\bigg\vert~ s_1 \right]\\
        & = \E\left[\sum_{t=1}^\infty \gamma^{t-1} \left(J^*+v^*(s_t) - v^*(s_{t+1})\right)  ~\bigg\vert~ s_1 \right]\\
        &= \frac{J^*}{1-\gamma} + v^*(s_1) - \E\left[\sum_{t=2}^\infty (\gamma^{t-2}-\gamma^{t-1})v^*(s_t) ~\bigg\vert~ s_1 \right] \\
        & \geq \frac{J^*}{1-\gamma} + \min_sv^*(s) - \max_sv^*(s)\sum_{t=2}^\infty (\gamma^{t-2}-\gamma^{t-1}) \\
        &= \frac{J^*}{1-\gamma} - \spn(v^*),
    \end{align*}
    where the first equality is by the Bellman equation for the undiscounted setting.
    
    Similarly, for the other direction, let $s_1, s_2, \cdots$ be generated by $\pi_\gamma$. We have
    \begin{align*}
        V^*(s_1) &= \E\left[\sum_{t=1}^\infty \gamma^{t-1} r(s_t, \pi_\gamma(s_t))  ~\bigg\vert~ s_1 \right]\\
        &\leq \E\left[\sum_{t=1}^\infty \gamma^{t-1} \left(J^*+v^*(s_t) - v^*(s_{t+1})\right)  ~\bigg\vert~ s_1 \right]\\
        &= \frac{J^*}{1-\gamma} + v^*(s_1) - \E\left[\sum_{t=2}^\infty (\gamma^{t-2}-\gamma^{t-1})v^*(s_t) ~\bigg\vert~ s_1 \right] \\
        & \leq \frac{J^*}{1-\gamma} + \max_sv^*(s) - \min_sv^*(s)\sum_{t=2}^\infty (\gamma^{t-2}-\gamma^{t-1}) \\
        &= \frac{J^*}{1-\gamma} + \spn(v^*), 
    \end{align*}
    where the first inequality is by sub-optimality of $\pi_\gamma$ for the undiscounted setting.
    \item Using previous part, for any $s_1, s_2 \in \mathcal{S}$, we have
    \begin{align*}
    |V^*(s_1) - V^*(s_2)| \leq \Big\vert V^*(s_1) - \frac{J^*}{1 - \gamma}\Big\vert + \Big\vert V^*(s_2) - \frac{J^*}{1 - \gamma}\Big\vert  \leq 2\spn(v^*). 
    \end{align*}
    Thus, $\spn (V^*) \leq 2\spn(v^*)$.
\end{enumerate}
\end{proof}

\subsection{Proof of Lemma \ref{lem: v-star minus q-star}}
\textbf{Lemma \ref{lem: v-star minus q-star}.} With probability at least $1 - \delta$,
\begin{align*}
    &\sum_{t=1}^T \left(V^*(s_t)- Q^*(s_t,a_t)\right) \leq 4HSA + 24\spn(v^*)\sqrt{HSAT \ln \frac{2T}{\delta}}.
\end{align*}  

\begin{proof}
We condition on the statement of Lemma \ref{lem: optimistic Q}, which happens with probability at least $1-\delta$.
    Let $n_t \geq 1$ denote $n_{t+1}(s_t,a_t)$, that is, the total number of visits to the state-action pair $(s_t,a_t)$ for the first $t$ rounds (including round $t$).  
    Also let $t_i(s,a)$ denote the timestep at which $(s,a)$ is visited the $i$-th time.
    Recalling the definition of $\alpha_{n_t}^i$ in Eq.~\eqref{eq: alpha-tau-i}, we have
    \begin{align}
    & \sum_{t=1}^T \left(\hat{V}_t(s_t)- V^*(s_t)\right) + \sum_{t=1}^T \left( V^*(s_t) - Q^*(s_t,a_t) \right) \\
    &=\sum_{t=1}^T \left(\hat{Q}_t(s_t,a_t)- Q^*(s_t,a_t)\right) \tag{because $a_t = \argmax_a \hat{Q}_t(s_t,a)$}\\
    &=\sum_{t=1}^T \left(\hat{Q}_{t+1}(s_t,a_t)- Q^*(s_t,a_t)\right) +   \sum_{t=1}^T \left(\hat{Q}_{t}(s_t,a_t)- \hat{Q}_{t+1}(s_t,a_t)\right)   \label{eqn:tmp two terms} \\ 
    &\leq 12\spn(v^*)\sum_{t=1}^T \sqrt{\frac{H}{n_t} \ln \frac{2T}{\delta}}
    + \gamma \sum_{t=1}^T \sum_{i=1}^{n_t}\alpha^i_{n_t}   
    \Big[\hat{V}_{t_i(s_t,a_t)}(s_{t_i(s_t,a_t)+1}) - V^*(s_{t_i(s_t,a_t)+1})\Big] + SAH \label{eq: tmp telescope}.
    \end{align}
    Here, we apply Lemma \ref{lem: optimistic Q} to bound the first term of Eq~.\eqref{eqn:tmp two terms} (note $\alpha_{n_t}^0 = 0$ by definition since $n_t \geq 1$), and also bound the second term of Eq~.\eqref{eqn:tmp two terms} by $SAH$ since for each fixed $(s,a)$, $\hat{Q}_t(s,a)$ is non-increasing in $t$ and overall cannot decrease by more than $H$ (the initial value). 
    
    To bound the third term of Eq.~\eqref{eq: tmp telescope} we write:
	\begin{align*}
	&\gamma \sum_{t=1}^T \sum_{i=1}^{n_t}\alpha^i_{n_t}   
    \Big[\hat{V}_{t_i(s_t,a_t)}(s_{t_i(s_t,a_t)+1}) - V^*(s_{t_i(s_t,a_t)+1})\Big] \\
    &= \gamma \sum_{t=1}^T \sum_{s, a} \mathbbm{1}_{[s_t=s, a_t=a]}\sum_{i=1}^{n_{t+1}(s, a)}\alpha^i_{n_{t+1}(s, a)}   
    \Big[\hat{V}_{t_i(s,a)}(s_{t_i(s,a)+1}) - V^*(s_{t_i(s,a)+1})\Big] \\
    &= \gamma \sum_{s, a}\sum_{j=1}^{n_{T+1}(s, a)} \sum_{i=1}^{j}\alpha^i_{j}   
    \Big[\hat{V}_{t_i(s,a)}(s_{t_i(s,a)+1}) - V^*(s_{t_i(s,a)+1})\Big].
    \end{align*}
    By changing the order of summation on $i$ and $j$, the latter is equal to
    \begin{align*}
    &\gamma \sum_{s, a}\sum_{i=1}^{n_{T+1}(s, a)} \sum_{j=i}^{n_{T+1}(s, a)}\alpha^i_{j}   
    \Big[\hat{V}_{t_i(s,a)}(s_{t_i(s,a)+1}) - V^*(s_{t_i(s,a)+1})\Big] \\
    &= \gamma \sum_{s, a}\sum_{i=1}^{n_{T+1}(s, a)}   
    \Big[\hat{V}_{t_i(s,a)}(s_{t_i(s,a)+1}) - V^*(s_{t_i(s,a)+1})\Big] \sum_{j=i}^{n_{T+1}(s, a)}\alpha^i_{j}
	\end{align*}
	Now, we can upper bound $\sum_{j=i}^{n_{T+1}(s, a)}\alpha^i_{j}$ by $\sum_{j=i}^\infty\alpha^i_{j}$ where the latter is equal to $1 + \frac{1}{H}$ by Lemma \ref{lem: learning rate}. Since $\hat{V}_{t_i(s,a)}(s_{t_i(s,a)+1}) - V^*(s_{t_i(s,a)+1}) \geq 0$ (by Lemma \ref{lem: optimistic Q}), we can write:
	\begin{align*}
	&\gamma \sum_{s, a}\sum_{i=1}^{n_{T+1}(s, a)}  
    \Big[\hat{V}_{t_i(s,a)}(s_{t_i(s,a)+1}) - V^*(s_{t_i(s,a)+1})\Big] \sum_{j=i}^{n_{T+1}(s, a)}\alpha^i_{j} \\
    &\leq \gamma \sum_{s, a}\sum_{i=1}^{n_{T+1}(s, a)}    
    \Big[\hat{V}_{t_i(s,a)}(s_{t_i(s,a)+1}) - V^*(s_{t_i(s,a)+1})\Big] \sum_{j=i}^{\infty}\alpha^i_{j} \\
    &= \gamma \sum_{s, a}\sum_{i=1}^{n_{T+1}(s, a)}   
    \Big[\hat{V}_{t_i(s,a)}(s_{t_i(s,a)+1}) - V^*(s_{t_i(s,a)+1})\Big] \left(1 + \frac{1}{H}\right) \\
        &=\left(1+\frac{1}{H}\right)\gamma \sum_{t=1}^T \left[ \hat{V}_{t}(s_{t+1}) - V^*(s_{t+1}) \right]\\
        &= \left(1+\frac{1}{H}\right)\gamma\sum_{t=1}^T \left[ \hat{V}_{t+1}(s_{t+1}) - V^*(s_{t+1}) \right] + \left(1+\frac{1}{H}\right) \sum_{t=1}^T \left[ \hat{V}_t(s_{t+1}) - \hat{V}_{t+1}(s_{t+1}) \right] \\
        &\leq \sum_{t=2}^{T+1} \left[ \hat{V}_{t}(s_{t}) - V^*(s_{t}) \right] +  \left(1+\frac{1}{H}\right) SH.  
    \end{align*}
    The last inequality is because $\left(1+\frac{1}{H}\right)\gamma \leq 1$ and that for any state $s$, $\hat{V}_t(s)\geq \hat{V}_{t+1}(s)$ and the value can decrease by at most $H$ (the initial value). Substituting in Eq.~\eqref{eq: tmp telescope} and telescoping with the left hand side, we have
    \begin{align*}
        \sum_{t=1}^T \left(V^*(s_t)- Q^*(s_t,a_t)\right) &\leq  12\spn(v^*)\sum_{t=1}^T \sqrt{\frac{H}{n_t} \ln \frac{2T}{\delta}} + \left(\hat{V}_{T+1}(s_{T+1}) - V^*(s_{T+1})\right) + \left(1+\frac{1}{H}\right) SH + SAH \\
        &\leq 12\spn(v^*)\sum_{t=1}^T \sqrt{\frac{H}{n_t} \ln \frac{2T}{\delta}} + 4SAH.
    \end{align*}
    Moreover, $\sum_{t=1}^T \frac{1}{\sqrt{n_t}} \leq 2\sqrt{SAT}$ because
    \begin{align*}
    \sum_{t=1}^T \frac{1}{\sqrt{n_{t+1}(s_t, a_t)}} &= \sum_{t=1}^T \sum_{s, a} \frac{\mathbbm{1}_{[s_t=s, a_t=a]}}{\sqrt{n_{t+1}(s, a)}} = \sum_{s, a}\sum_{j=1}^{n_{T+1}(s, a)}\frac{1}{\sqrt j} \leq \sum_{s, a} 2 \sqrt{n_{T+1}(s, a)} \leq 2 \sqrt{SA \sum_{s, a}n_{T+1}(s, a)} = 2\sqrt{SAT},
    \end{align*}
    where the last inequality is by Cauchy-Schwarz inequality. This finishes the proof.
\end{proof}

\begin{lemma}
\label{lem: optimistic Q}
With probability at least $1-\delta$, for any $t = 1, \ldots, T$ and state-action pair $(s,a)$, the following holds
\begin{align*}
    0 &\leq 
    \hat{Q}_{t+1}(s,a) - Q^*(s,a) \leq  H\alpha_\tau^0 + \gamma \sum_{i=1}^\tau \alpha_\tau^i \left[\hat{V}_{t_i}(s_{t_i+1}) - V^*(s_{t_i+1}) \right] + 12\spn(v^*)\sqrt{\frac{H}{\tau} \ln \frac{2T}{\delta}},
\end{align*}
where $\tau=n_{t+1}(s,a)$ (i.e., the total number of visits to $(s,a)$ for the first $t$ timesteps), $\alpha_\tau^i$ is defined by \eqref{eq: alpha-tau-i},
and $t_1, \ldots, t_\tau \leq t$ are the timesteps on which $(s,a)$ is taken.
\end{lemma}
\begin{proof}
    Recursively substituting $Q_t(s, a)$ in Eq.~\eqref{eqn: q update} of the algorithm, we have
    \begin{align*}
        Q_{t+1}(s,a) = H\alpha_\tau^0 + \sum_{i=1}^\tau \alpha^i_\tau \left[r(s,a) + \gamma\hat{V}_{t_i}(s_{t_i+1})\right] + \sum_{i=1}^\tau \alpha_\tau^i b_i.  
    \end{align*}
    Moreover, since $\sum_{i=1}^\tau \alpha_\tau^i = 1$ (Lemma \ref{lem: learning rate}), By Bellman equation we have
    \begin{align*}
        Q^*(s,a) = \alpha_\tau^0 Q^*(s,a) + \sum_{i=1}^\tau \alpha_\tau^i \left[ r(s,a) + \gamma \E_{s'\sim p(\cdot|s,a)}V^*(s') \right].
    \end{align*}
    Taking their difference and adding and subtracting a term $\gamma\sum_{i=1}^\tau \alpha_\tau^i V^*(s_{t_i+1})$ lead to: 
    \begin{align*}
        Q_{t+1}(s,a) - Q^*(s,a) = \alpha^0_\tau\left(H- Q^*(s,a)\right) &+ \gamma \sum_{i=1}^\tau \alpha_\tau^i \left[ \hat{V}_{t_i}\left(s_{t_i+1}\right) - V^*\left(s_{t_i+1}\right) \right] \\
        &+ \gamma \sum_{i=1}^\tau \alpha_\tau^i \left[V^*(s_{t_i+1}) - \E_{s'\sim p(\cdot | s,a)} V^*(s') \right] + \sum_{i=1}^\tau \alpha_\tau^i b_i. 
    \end{align*}
    The first term is upper bounded by $\alpha_\tau^0 H$ clearly and lower bounded by $0$ since $Q^*(s,a) \leq \sum_{i=0}^\infty \gamma^i = \frac{1}{1-\gamma} = H$.

    The third term is a martingale difference sequence with each term bounded in $[-\gamma \alpha_\tau^i\spn(V^*), \gamma\alpha_\tau^i \spn(V^*)]$. Therefore, by Azuma's inequality (Lemma \ref{lem: azuma}), its absolute value is bounded by $\gamma \spn(V^*)\sqrt{2\sum_{i=1}^\tau (\alpha_\tau^i)^2\ln \frac{2T}{\delta}} \leq 2\gamma \spn(V^*)\sqrt{\frac{H}{\tau}\ln \frac{2T}{\delta}} \leq 4\gamma \spn(v^*)\sqrt{\frac{H}{\tau}\ln \frac{2T}{\delta}}$ with probability at least $1 - \frac{\delta}{T}$, where the first inequality is by Lemma \ref{lem: learning rate} and the last inequality is by Lemma \ref{lem: discounted average connection}. 
Note that when $t$ varies from $1$ to $T$ and $(s,a)$ varies over all possible state-action pairs, the third term only takes $T$ different forms.
Therefore, by taking a union bound over these $T$ events, we have: with probability $1-\delta$, the third term is bounded by $4\gamma \spn(v^*)\sqrt{\frac{H}{\tau}\ln \frac{2T}{\delta}}$ in absolute value for all $t$ and $(s,a)$.
    
The forth term is lower bounded by $4 \spn(v^*)\sqrt{\frac{H}{\tau}\ln \frac{2T}{\delta}}$ and upper bounded by $8 \spn(V^*)\sqrt{\frac{H}{\tau}\ln \frac{2T}{\delta}}$, by Lemma \ref{lem: learning rate}.

Combining all aforementioned upper bounds and the fact $\hat{Q}_{t+1}(s,a) = \min\left\{ \hat{Q}_{t}(s,a), Q_{t+1}(s,a) \right\} \leq Q_{t+1}(s,a)$ we prove the upper bound in the lemma statement.
To prove the lower bound, further note that the second term can be written as $\gamma \sum_{i=1}^\tau \alpha_\tau^i \left[ \max_a\hat{Q}_{t_i}(s_{t_i+1},a) - \max_a Q^*(s_{t_i+1},a) \right]$.
Using a direct induction with all aforementioned lower bounds and the fact $\hat{Q}_{t+1}(s,a) = \min\left\{ \hat{Q}_{t}(s,a), Q_{t+1}(s,a) \right\}$ we prove the lower bound in the lemma statement as well.
%
%
\end{proof}

\subsection{Proof of Lemma \ref{lem: discounted bellman}}
\textbf{Lemma \ref{lem: discounted bellman}.} With probability at least $1 - \delta$,
\begin{align*}
&\sum_{t=1}^T \left(Q^*(s_t, a_t) - \gamma V^*(s_t) - r(s_t, a_t)\right) \leq 2 \spn(v^*) \sqrt{2T \ln \frac{1}{\delta}} + 2 \spn(v^*).
\end{align*}

\begin{proof}
By Bellman equation for the discounted problem, we have $Q^*(s_t, a_t) - \gamma V^*(s_t) - r(s_t, a_t) = \gamma \left(\E_{s' \sim p(\cdot | s_t, a_t)} [V^*(s')] - V^*(s_t)\right)$. Adding and subtracting $V^*(s_{t+1})$ and summing over $t$ we will get
\begin{align*}
\sum_{t=1}^T &\left(Q^*(s_t, a_t) - \gamma V^*(s_t) - r(s_t, a_t)\right) = \gamma \sum_{t=1}^T \left(\E_{s' \sim p(\cdot | s_t, a_t)} [V^*(s')] - V^*(s_{t+1})\right) + \gamma \sum_{t=1}^T \left(V^*(s_{t+1}) - V^*(s_t)\right)
\end{align*}
The summands of the first term on the right hand side constitute a martingale difference sequence. Thus, by Azuma's inequality (Lemma \ref{lem: azuma}) and the fact that $\spn(V^*) \leq 2 \spn(v^*)$ (Lemma \ref{lem: discounted average connection}), this term is upper bounded by $2\gamma \spn(v^*) \sqrt{2T \ln \frac{1}{\delta}}$, with probability at least $1 - \delta$. The second term is equal to $\gamma (V^*(s_{T+1}) - V^*(s_1))$ which is upper bounded by $2\gamma \spn(v^*)$. Recalling $\gamma < 1$ completes the proof.
\end{proof}


\section{Omitted Proofs in Section~\ref{section: policy optimization} --- Proofs for Lemma~\ref{lemma: PG unbiasedness} and Lemma~\ref{lemma: stability lemma}}
\subsection{Auxiliary Lemmas}
In this subsection, we state several lemmas that will be helpful in the analysis. 
\begin{lemma}[{\citep[Section 4.5]{levin2017markov}}]
    \label{lemma: t mix epsilon lemma}
    Define 
    \begin{align*}
     \mix(\epsilon):=\max_{\pi}\min\left\{ t\geq 1~\Big\vert ~ \| (P^{\pi})^t(s,\cdot) - \mu^\pi \|_1\leq \epsilon, \forall s \right\}, 
    \end{align*}
    so that $\mix = \mix(\frac{1}{4})$. We have
    \begin{align*}
         \mix(\epsilon)\leq \left\lceil \log_2 \frac{1}{\epsilon} \right\rceil\mix
    \end{align*}
for any $\epsilon\in (0, \frac{1}{2}]$. 
\end{lemma}

\begin{corollary}
    \label{lemma: mixing time lemma}
     For an ergodic MDP with mixing time $\mix$, we have
    \begin{align*}
        \|(P^{\pi})^t(s, \cdot) - \mu^{\pi}\|_1 \leq 2\cdot 2^{- \frac{t}{\mix} }, \qquad \qquad \forall \pi, s
    \end{align*}
    for all $\pi$ and all $t\geq 2\mix$. 
\end{corollary}
\begin{proof}
    Lemma~\ref{lemma: t mix epsilon lemma} implies for any $\epsilon\in (0,\frac{1}{2}]$, as long as $t\geq \lceil \log_2(1/\epsilon) \rceil \mix$, we have
    \begin{align*}
        \| (P^{\pi})^t(s,\cdot) - \mu^\pi \|_1\leq \epsilon. 
    \end{align*}
    This condition can be satisfied by picking $\log_2(1/\epsilon)=\frac{t}{\mix}-1$, which leads to $\epsilon=2\cdot 2^{-\frac{t}{\mix}}$. 
\end{proof}

\begin{corollary}
\label{corollary: tail sum mixing time}
     Let $N=4\mix \log_2 T$.  For an ergodic MDP with mixing time $\mix < T/4$, we have for all $\pi$:
     \begin{align*}
          \sum_{t=N}^\infty \|(P^{\pi})^t(s, \cdot) - \mu^{\pi}\|_1 \leq \frac{1}{T^3}. 
     \end{align*}
\end{corollary}
\begin{proof}
    By Corollary~\ref{lemma: mixing time lemma}, 
    \begin{align*}
         \sum_{t=N}^\infty \|(P^{\pi})^t(s, \cdot) - \mu^{\pi}\|_1 \leq  \sum_{t=N}^\infty 2\cdot 2^{-\frac{t}{\mix}} = \frac{2\cdot 2^{-\frac{N}{\mix}}}{1-2^{-\frac{1}{\mix}}}\leq \frac{2\mix}{\ln 2}\cdot 2\cdot 2^{-\frac{N}{\mix}} = \frac{2\mix}{\ln 2}\cdot 2\cdot \frac{1}{T^4} \leq \frac{1}{T^3}. 
    \end{align*}
\end{proof}

\begin{lemma}[Stated in \citep{wang2017primal} without proof]
\label{lemma: bounded span of pi}
For an ergodic MDP with mixing time $\mix$, and any $\pi, s, a$, 
\begin{align*}
      |v^\pi(s)| &\leq 5\mix, \\
      |q^\pi(s,a)| &\leq 6\mix. 
\end{align*}
\end{lemma}
\begin{proof}
Using the identity of Eq.~\eqref{eqn:v_pi_identity} we have
\begin{align*}      
\left\lvert  v^{\pi}(s)  \right\rvert 
&= \left\lvert  \sum_{t=0}^{\infty} ((P^\pi)^t(s, \cdot)-\mu^\pi)^\top r^\pi  \right\rvert \\
&\leq  \sum_{t=0}^{\infty} \left\|(P^\pi)^t(s, \cdot)-\mu^\pi\right\|_1  \|r^\pi\|_\infty \\
&\leq  \sum_{t=0}^{2\mix-1}  \left\|(P^\pi)^t(s, \cdot)-\mu^\pi\right\|_1 + \sum_{i=2}^\infty \sum_{t=i \mix}^{(i+1)\mix-1}  \left\|(P^\pi)^t(s, \cdot)-\mu^\pi\right\|_1 \\
&\leq 4\mix + \sum_{i=2}^\infty 2\cdot 2^{-i} \mix  \tag{by $\left\|(P^\pi)^t(s, \cdot)-\mu^\pi\right\|_1 \leq 2$ and Corollary~\ref{lemma: mixing time lemma}}\\
&\leq 5\mix,
\end{align*}
and thus
\begin{align*}
      \left\lvert q^{\pi}(s,a) \right\rvert = \left\lvert r(s,a) + \E_{s'\sim p(\cdot|s,a)}[v^\pi(s')] \right\rvert \leq 1 + 5\mix \leq 6\mix. 
\end{align*}
\end{proof}


\begin{lemma}[{\citep[Lemma 2]{neu2013online}}] 
\label{lemma: reward diff lemma}
For any two policies $\pi, \tilde{\pi}$, 
\begin{align*}
      J^{\tilde{\pi}} - J^{\pi} = \sum_{s} \sum_{a} \mu^{\tilde{\pi}}(s)\left(\tilde{\pi}(a|s)-\pi(a|s)\right) q^{\pi}(s,a). 
\end{align*}
\end{lemma}
\begin{proof}
Using Bellman equation we have
      \begin{align*}
          &\sum_{s} \sum_{a} \mu^{\tilde{\pi}}(s)\tilde{\pi}(a|s) q^{\pi}(s,a) \\
          &= \sum_{s} \sum_{a} \mu^{\tilde{\pi}}(s)\tilde{\pi}(a|s) \left(r(s,a)-J^{\pi}+\sum_{s'}p(s'|s,a)v^\pi(s')\right)\\
          &= J^{\tilde{\pi}}-J^\pi + \sum_{s'}\mu^{\tilde{\pi}}(s')v^{\pi}(s') \\
          &= J^{\tilde{\pi}}-J^\pi + \sum_{s}\mu^{\tilde{\pi}}(s)v^{\pi}(s) \\
          &= J^{\tilde{\pi}}-J^\pi + \sum_{s}\sum_a \mu^{\tilde{\pi}}(s)\pi(a|s)q^\pi(s,a), 
      \end{align*}
      where the second equality uses the facts $ J^{\tilde{\pi}} = \sum_{s} \sum_{a} \mu^{\tilde{\pi}}(s)\tilde{\pi}(a|s) r(s,a)$ and $ \sum_{s,a}\mu^{\tilde{\pi}}(s)\tilde{\pi}(a|s)p(s'|s,a)=\mu^{\tilde{\pi}}(s')$. Rearranging gives the desired equality. 
\end{proof}

\begin{lemma}
\label{lemma: visit at least once}
      Let $\calI=\{t_1+1, t_1+2, \ldots, t_2\}$ be a certain period of an episode $k$ of Algorithm~\ref{alg: PG-online-2} with $|\calI|\geq N=4\mix \log_2 T$. Then for any $s$, the probability that the algorithm never visits $s$ in $\calI$ is upper bounded by 
\begin{align*}
      \left(1-\frac{3\mu^{\pi_k}(s)}{4}\right)^{\left\lfloor \frac{|\calI|}{N} \right\rfloor}.
\end{align*}
\end{lemma}
\begin{proof}
     Consider a subset of $\calI$: $\{t_1+N, t_1+2N, \ldots\}$ which consists of at least $\left\lfloor \frac{t_2-t_1}{N} \right\rfloor$ rounds that are at least $N$-step away from each other. By Corollary~\ref{lemma: mixing time lemma}, we have for any $i$,
\begin{align*}
     \Big\lvert \Pr[s_{t_1+iN}=s ~\big|~ s_{t_1+(i-1)N}] - \mu^{\pi_k}(s)\Big\rvert \leq 2\cdot 2^{-\frac{N}{\mix}} \leq 2\cdot 2^{-4\log_2 T}\leq \frac{2}{T^4}, 
\end{align*}
that is, conditioned on the state at time $t_1+(i-1)N$, the state distribution at time $t_1+iN$ is close to the stationary distribution induced by $\pi_k$. 
Therefore we further have $\Pr[s_{t_1+iN}=s ~\big|~ s_{t_1+(i-1)N}]\geq \mu^{\pi_k}(s)-\frac{2}{T^4}\geq \frac{3}{4} \mu^{\pi_k}(s)$, where the last step uses the fact $\mu^{\pi_k}(s) \geq \frac{1}{\hit} \geq \frac{4}{T}$.
The probability that the algorithm does not visit $s$ in any of the rounds $\{t_1+N, t_1+2N, \ldots\}$ is then at most
\begin{align*}
    \left(1-\frac{3\mu^{\pi_k}(s)}{4}\right)^{\left\lfloor \frac{t_2-t_1}{N} \right\rfloor}=\left(1-\frac{3\mu^{\pi_k}(s)}{4}\right)^{\left\lfloor \frac{|\calI|}{N} \right\rfloor},
\end{align*}
finishing the proof.
\end{proof}

\subsection{Proof for Lemma~\ref{lemma: PG unbiasedness}}

\begin{proof}[Proof for Eq.\eqref{eqn: unbiased mean}]
    In this proof, we consider a specific episode $k$ and a specific state $s$. For notation simplicity, we use $\pi$ for $\pi_k$ throughout this proof, and all the expectations or probabilities are conditioned on the history before episode $k$. Suppose that when Algorithm~\ref{alg: PG-online-2} calls \estimate in episode $k$ for state $s$, it finds $M$ disjoint intervals that starts from $s$. Denote the reward estimators corresponding to the $i$-th interval as $\widehat{\beta}_{k,i}(s, \cdot)$ (i.e., the $y_i(\cdot)$ in Algorithm~\ref{alg: estimate}), and the time when the $i$-th interval starts as $\tau_i$ (thus $s_{\tau_i}=s$). Then by the algorithm, we have 
\begin{align}
    \widehat{\beta}_k(s,a) = \begin{cases}
         \frac{\sum_{i=1}^M \widehat{\beta}_{k,i}(s, a)}{M}    &\text{if\ } M>0,\\
         0 &\text{if\ } M=0.    \label{eqn: definition of q hat}
    \end{cases}  
\end{align}
Since each $\widehat{\beta}_{k,i}(s,a)$ is constructed by a length-$(N+1)$ trajectory starting from $s$ at time $\tau_i\leq kB-N$, we can calculate its \textit{conditional expectation} as follows: 
\begin{align}
     &\E\left[  \widehat{\beta}_{k,i}(s,a)   \Big\vert  s_{\tau_i}=s  \right]\nonumber \\
     &=\Pr[a_{\tau_i}=a~|~s_{\tau_i}=s] \times \frac{r(s,a) + \E\left[ \sum_{t=\tau_i+1}^{\tau_i+N} r(s_t,a_t) ~\Big|~ (s_{\tau_i},a_{\tau_i})=(s,a) \right]}{\pi(a|s)}\nonumber \\
     &=  r(s,a) + \sum_{s'} p(s'|s,a) \E\left[\sum_{t=\tau_i+1}^{\tau_i+N} r(s_t,a_t) \Big\vert s_{\tau_i+1}=s'\right] \nonumber\\
     &= r(s,a) + \sum_{s'} p(s'|s,a) \sum_{j=0}^{N-1} \mathbf{e}_{s'}^\top (P^{\pi})^j r^{\pi}\nonumber \\
     &= r(s,a) + \sum_{s'} p(s'|s,a) \sum_{j=0}^{N-1} (\mathbf{e}_{s'}^\top (P^{\pi})^j - (\mu^{\pi})^\top) r^{\pi}  + NJ^{\pi} \tag{because $ \mu^{\pi \top} r^{\pi} = J^{\pi}$}\nonumber\\
        &=r(s,a) + \sum_{s'} p(s'|s,a)  v^{\pi}(s')  + NJ^{\pi} -  \sum_{s'} p(s'|s,a) \sum_{j=N}^{\infty} (\mathbf{e}_{s'}^\top (P^{\pi})^j - (\mu^{\pi})^\top) r^{\pi}\tag{By Eq.~\eqref{eqn:v_pi_identity}} \nonumber \\
        &= q^{\pi}(s,a) + NJ^\pi - \delta(s,a) \nonumber \\
        &= \beta^{\pi}(s,a) - \delta(s,a),  \label{eqn: unbiased calculation}
    \end{align}
    where $\delta(s,a)\triangleq  \sum_{s'} p(s'|s,a) \sum_{j=N}^{\infty} (\mathbf{e}_{s'}^\top (P^{\pi})^j - (\mu^{\pi})^\top) r^{\pi}$. 
    By Corollary~\ref{corollary: tail sum mixing time}, 
    \begin{align}
        \left\lvert  \delta(s,a) \right\rvert \leq \frac{1}{T^3}.   \label{eqn: small error calculation}
    \end{align}
    Thus, 
    \begin{align*}
        \Bigg\vert\E\left[  \widehat{\beta}_{k,i}(s,a)   \Big\vert  s_{\tau_i}=s  \right] -  \beta^\pi(s,a) \Bigg\vert \leq \frac{1}{T^3}. 
    \end{align*}
    This shows that $\widehat{\beta}_{k,i}(s,a)$ is an \textit{almost} unbiased estimator for $\beta^{\pi}$ conditioned on all history before $\tau_i$. Also, by our selection of the episode length, $M>0$ will happen with very high probability according to Lemma~\ref{lemma: visit at least once}. These facts seem to indicate that $\widehat{\beta}_k(s,a)$ -- an average of several $\widehat{\beta}_{k,i}(s,a)$ -- will also be an almost unbiased estimator for $\beta^{\pi}(s,a)$ with small error.

However, a caveat here is that the quantity $M$ in Eq.\eqref{eqn: definition of q hat} is random, and it is not independent from the reward estimators $\sum_{i=1}^M \widehat{\beta}_{k,i}(s,a)$. Therefore, to argue that the expectation of $\E[\widehat{\beta}_k(s,a)]$ is close to $\beta^{\pi}(s,a)$, more technical work is needed. 
Specifically, we use the following two steps to argue that $\E[\widehat{\beta}_k(s,a)]$ is close to $\beta^{\pi}(s,a)$. \\
\textbf{Step 1.} Construct an \textit{imaginary world} where $\widehat{\beta}_k(s,a)$ is an almost unbiased estimator of $\beta^{\pi}(s,a)$. \\
\textbf{Step 2.} Argue that the expectation of $\widehat{\beta}_k(s,a)$ in the real world and the expectation of $\widehat{\beta}_k(s,a)$ in the imaginary world are close. 

\begin{figure}[h]
\includegraphics[width=\textwidth]{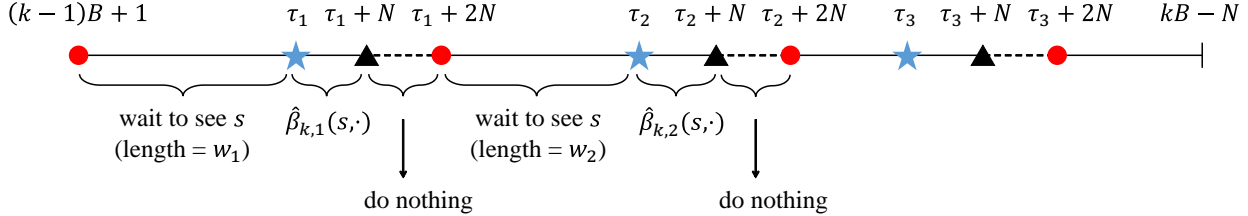}
\caption{An illustration for the sub-algorithm \estimate with target state $s$ (best viewed in color). The red round points indicate that the algorithm ``starts to wait'' for a visit to $s$. When the algorithm reaches $s$ (the blue stars) at time $\tau_i$, it starts to record the sum of rewards in the following $N+1$ steps, i.e. $\sum_{t=\tau_i}^{\tau_i+N} r(s_t,a_t)$. This is used to construct $\widehat{\beta}_{k,i}(s, \cdot)$. The next point the algorithm ``starts to wait for $s$'' would be $\tau_i+2N$ if this is still no later than $kB-N$. }
\label{figure: illustrate estimateQ}
\end{figure}

\paragraph{Step 1. }
We first examine what \estimate sub-algorithm does in an episode $k$ for a state $s$. The goal of this sub-algorithm is to collect disjoint intervals of length $N+1$ that start from $s$, calculate a reward estimator from each of them, and finally average the estimators over all intervals to get a good estimator for $\beta^{\pi}(s,\cdot)$. However, after our algorithm collects an interval $[\tau, \tau+N]$, it \textit{rests} for another $N$ steps before starting to find the next visit to $s$ -- i.e., it restarts from $\tau+2N$ (see Line~\ref{line: skip 2N} in \estimate (Algorithm~\ref{alg: estimate}), and also the illustration in Figure~\ref{figure: illustrate estimateQ}). 

The goal of doing this is to de-correlate the observed reward and the number of collected intervals: as shown in Eq.\eqref{eqn: definition of q hat}, these two quantities affect the numerator and the denominator of $\widehat{\beta}_k(s,\cdot)$ respectively, and if they are highly correlated, then $\widehat{\beta}_k(s,\cdot)$ may be heavily biased from $\beta^{\pi}(s,\cdot)$.  On the other hand, if we introduce the ``rest time'' after we collect each interval (i.e., the dashed segments in Figure~\ref{figure: illustrate estimateQ}), then since the length of the rest time ($N$) is longer than the mixing time, the process will almost totally ``forget'' about the reward estimators collected before. In Figure~\ref{figure: illustrate estimateQ}, this means that the state distributions at the red round points (except for the left most one) will be close to $\mu^{\pi}$ when conditioned on all history that happened $N$ rounds ago. 

We first argue that if the process can indeed ``reset its memory'' at those red round points in Figure~\ref{figure: illustrate estimateQ} (except for the left most one), then we get almost unbiased estimators for $\beta^\pi(s,\cdot)$. That is, consider a process like in Figure~\ref{figure: illustrate estimateQ reset} where everything remains same as in \estimate except that after every rest interval, the state distribution is directly reset to the stationary distribution $\mu^\pi$. 

\begin{figure}[h]
\includegraphics[width=\textwidth]{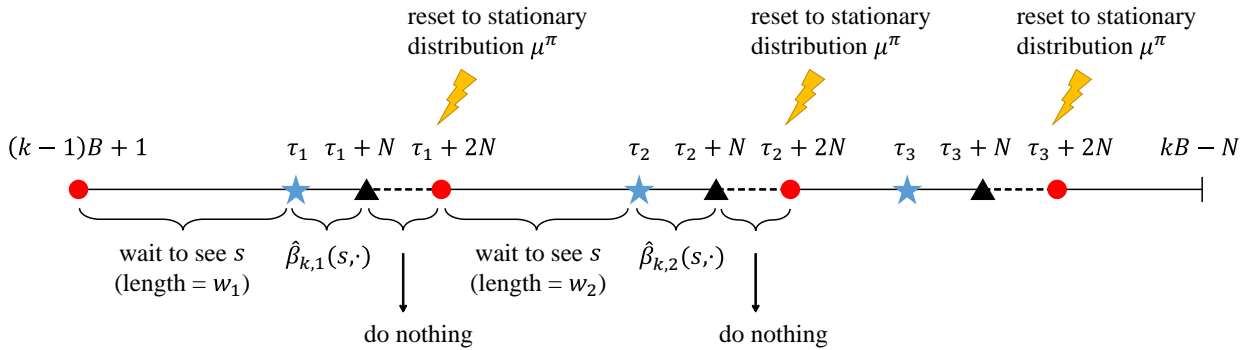}
\caption{The imaginary world (best viewed in color)}
\label{figure: illustrate estimateQ reset}
\end{figure}

Below we calculate the expectation of $\widehat{\beta}_k(s,a)$ in this imaginary world. As specified in Figure~\ref{figure: illustrate estimateQ reset}, we use $\tau_i$ to denote the $i$-th time \estimate starts to record an interval (therefore $s_{\tau_i}=s$), and let $w_i=\tau_i-(\tau_{i-1}+2N)$ for $i>1$ and $w_1=\tau_1-((k-1)B+1)$ be the ``wait time'' before starting the $i$-th interval. Note the following facts in the imaginary world: 
\begin{enumerate}
    \item $M$ is determined by the sequence $w_1, w_2, \ldots$ because all other segments in the figures have fixed length. 
    \item $w_1$ only depends on $s_{(k-1)B+1}$ and $P^\pi$, and $w_i$ only depends on the stationary distribution $\mu^\pi$ and $P^\pi$ because of the reset.  
\end{enumerate}
The above facts imply that in the imaginary world, $w_1, w_2, \ldots$, as well as $M$, are all independent from $\widehat{\beta}_{k,1}(s,a), \widehat{\beta}_{k,2}(s,a), \ldots$. Let $\E'$ denote the expectation in the imaginary world. Then
\begin{align}
    \E'\left[ \widehat{\beta}_k(s,a) \right] 
    &= \Pr[w_1\leq B-N]\times \E'_{\{w_i\}} \left[ \frac{1}{M} \sum_{i=1}^M  \E'\left[\widehat{\beta}_{k,i}(s,a) \Big\vert \{w_i\}  \right]  \Bigg\vert w_1\leq B-N \right] + \Pr[w_1 > B-N] \times 0 \nonumber  \\
    &= \Pr[w_1\leq B-N]\times \E'_{\{w_i\}} \left[ \frac{1}{M}\left( \sum_{i=1}^M \beta^\pi(s,a) - \delta(s,a) \right) \right]  \tag{by the same calculation as in \eqref{eqn: unbiased calculation}}\nonumber  \\
    &= \Pr[w_1\leq B-N]\times \left(\beta^\pi(s,a) - \delta(s,a) \right)   \nonumber \\
    &= \beta^{\pi}(s,a) - \delta'(s,a),  \label{eqn: imaginary world no bias}
\end{align}
where $\E'_{\{w_i\}}$ denotes the expectation over the randomness of $w_1, w_2, \ldots$, and $\delta'(s,a) = (1-\Pr[w_1\leq B-N])\left(\beta^\pi(s,a) - \delta(s,a) \right)+\delta(s,a)$. By Lemma~\ref{lemma: visit at least once}, we have $\Pr[w_1\leq B-N]\geq 1-\left(1-\frac{3}{4\hit}\right)^{\frac{B-N}{N}}  = 1-\left(1-\frac{3}{4\hit}\right)^{4\hit\log_2 T-1}\geq 1-\frac{1}{T^3}$. Together with Eq.~\eqref{eqn: small error calculation} and Lemma~\ref{lemma: bounded span of pi}, we have
\[
|\delta'(s,a)|\leq \frac{1}{T^3}(|\beta^\pi(s,a) |+|\delta(s,a)|) + |\delta(s,a)|
\leq \frac{1}{T^3}(6\mix+ N + \frac{1}{T^3}) + \frac{1}{T^3} =
\order\left(\frac{1}{T^2}\right),
\]
and thus
\begin{equation}\label{eqn:imaginary_world_bias}
\left|\E'\left[ \widehat{\beta}_k(s,a) \right]  - \beta^{\pi}(s,a)\right| = \order\left(\frac{1}{T^2}\right).
\end{equation}

\paragraph{Step 2.} 
Note that $\widehat{\beta}_k(s,a)$ is a deterministic function of $X = \left(M, \tau_1, \calT_1, \tau_2, \calT_2, \ldots, \tau_M, \calT_M \right)$, where $\calT_i=(a_{\tau_i}, s_{\tau_i+1}, a_{\tau_i+1}, \ldots, s_{\tau_i+N}, a_{\tau_i+N})$. We use $\widehat{\beta}_k(s,a) = f(X)$ to denote this mapping. To say $\E[\widehat{\beta}_k(s,a)]$ and $\E'[\widehat{\beta}_k(s,a)]$ are close, we bound their ratio: 
\begin{align}
     \frac{\E[\widehat{\beta}_k(s,a)]}{\E'[\widehat{\beta}_k(s,a)]} = \frac{\sum_X f(X)\mathbb{P}(X) }{\sum_X f(X)\mathbb{P}'(X)}\leq \max_X \frac{\mathbb{P}(X)}{\mathbb{P}'(X)},  \label{eqn: bound E ratio}
\end{align}
where we use $\mathbb{P}$ and $\mathbb{P}'$ to denote the probability mass function in the real world and the imaginary world respectively, and in the last inequality we use the non-negativeness of $f(X)$. 

For a fixed sequence of $X$, the probability of generating $X$ in the real world is 
\begin{align}
      \mathbb{P}(X) 
      &= \mathbb{P}(\tau_1) \times \mathbb{P}(\calT_1|\tau_1) \times \mathbb{P}(\tau_2|\tau_1, \calT_1)\times \mathbb{P}(\calT_2 | \tau_2) \times \cdots \times \mathbb{P}(\tau_M|\tau_{M-1}, \calT_{M-1}) \nonumber \\
      &\qquad \qquad \qquad \qquad \qquad \times \mathbb{P}(\calT_M|\tau_M)\times \Pr\left[s_t\neq s,\ \forall t\in [\tau_M+2N, kB-N]\Big\vert \tau_M,\calT_M\right]. \label{eqn: compare PX 1}
\end{align} 
In the imaginary world, it is 
\begin{align}
      \mathbb{P}'(X) 
      &= \mathbb{P}(\tau_1) \times \mathbb{P}(\calT_1|\tau_1) \times \mathbb{P}'(\tau_2|\tau_1, \calT_1)\times \mathbb{P}(\calT_2 | \tau_2) \times \cdots \times \mathbb{P}'(\tau_M|\tau_{M-1}, \calT_{M-1})  \nonumber \\
      &\qquad \qquad \qquad \qquad \qquad \times \mathbb{P}(\calT_M|\tau_M)\times \Pr\left[s_t\neq s,\ \forall t\in [\tau_M+2N, kB-N]\Big\vert \tau_M,\calT_M\right].  \label{eqn: compare PX 2}
\end{align} 
Their difference only comes from $\mathbb{P}(\tau_{i+1}|\tau_{i}, \calT_{i})\neq \mathbb{P}'(\tau_{i+1}|\tau_{i}, \calT_{i})$ because of the reset. Note that 
\begin{align}
      \mathbb{P}(\tau_{i+1}|\tau_{i}, \calT_{i})& = \sum_{s'\neq s} \mathbb{P}(s_{\tau_{i}+2N}=s'|\tau_{i}, \calT_{i}) \times \Pr\left[s_t\neq s,\ \forall t\in[\tau_{i}+2N, \tau_{i+1}-1], s_{\tau_{i+1}}=s ~\Big\vert s_{\tau_{i}+2N}=s'\right], \label{eqn: comparing P tau P' tau 1} \\
       \mathbb{P}'(\tau_{i+1}|\tau_{i}, \calT_{i})& = \sum_{s'\neq s} \mathbb{P}'(s_{\tau_{i}+2N}=s'|\tau_{i}, \calT_{i}) \times \Pr\left[s_t\neq s,\ \forall t\in[\tau_{i}+2N, \tau_{i+1}-1], s_{\tau_{i+1}}=s ~\Big\vert s_{\tau_{i}+2N}=s'\right]. \label{eqn: comparing P tau P' tau 2}
\end{align}
Because of the reset in the imaginary world, $\mathbb{P}'(s_{\tau_{i}+2N}=s'|\tau_{i}, \calT_{i}) = \mu^{\pi}(s')$ for all $s'$; in the real world, since at time $\tau_{i}+2N$, the process has proceeded $N$ steps from $\tau_{i}+N$ (the last step of $\calT_{i}$), by Corollary~\ref{lemma: mixing time lemma} we have 
\begin{align*}
      \frac{\mathbb{P}(s_{\tau_{i}+2N}=s'|\tau_{i}, \calT_{i})}{\mathbb{P}'(s_{\tau_{i}+2N}=s'|\tau_{i}, \calT_{i})} = 1+ \frac{\mathbb{P}(s_{\tau_{i}+2N}=s'|\tau_{i}, \calT_{i})-\mu^{\pi}(s')}{\mu^{\pi}(s')}  \leq  1 +\frac{2}{T^4 \mu^{\pi}(s') }  \leq 1+ \frac{1}{T^3} \quad \text{\ for all\ } s', 
\end{align*}
which implies $\frac{\mathbb{P}(\tau_{i+1}|\tau_{i}, \calT_{i})}{\mathbb{P}'(\tau_{i+1}|\tau_{i}, \calT_{i})}\leq 1+\frac{1}{T^3}$ by \eqref{eqn: comparing P tau P' tau 1} and \eqref{eqn: comparing P tau P' tau 2} . This further implies $\frac{\mathbb{P}(X)}{\mathbb {P}'(X)}\leq \left(1+\frac{1}{T^3}\right)^M\leq e^{\frac{M}{T^3}}\leq e^{\frac{1}{T^2}} \leq 1+\frac{2}{T^2}$ by \eqref{eqn: compare PX 1} and \eqref{eqn: compare PX 2}. From \eqref{eqn: bound E ratio}, we then have 
\begin{align*}
    \frac{\E[\widehat{\beta}_k(s,a)]}{\E'[\widehat{\beta}_k(s,a)]} \leq 1+\frac{2}{T^2}. 
\end{align*}
Thus, using the bound from Eq.~\eqref{eqn:imaginary_world_bias} we have
\[
\E[\widehat{\beta}_k(s,a)]\leq \left(1+\frac{2}{T^2}\right)\E'[\widehat{\beta}_k(s,a)] \leq \left(1+\frac{2}{T^2}\right)\left(\beta_k(s,a) + \order\left(\frac{1}{T^2}\right)\right) \leq \beta_k(s,a) +\order\left(\frac{1}{T}\right). 
\]
Similarly we can prove the other direction: $\beta_k(s,a) \leq \E[\widehat{\beta}_k(s,a)] + \order\left(\frac{1}{T}\right)$, finishing the proof. 
\end{proof}

\begin{proof}[Proof for Eq.\eqref{eqn: unbiased variance}]
    We use the same notations, and the similar approach as in the previous proof for Eq.~\eqref{eqn: unbiased mean}. That is, we first bound the expectation of the desired quantity in the imaginary world, and then argue that the expectation in the imaginary world and that in the real world are close. 
\paragraph{Step 1. } 
Define $\Delta_i = \widehat{\beta}_{k,i}(s,a) -\beta^\pi(s,a) +\delta(s,a)$. Then $\E'[\Delta_i~|~\{w_i\}]=0$ by Eq.\eqref{eqn: unbiased calculation}. Thus in the imaginary world, 
\begin{align}
    &\E'\left[ \left(\widehat{\beta}_k(s,a) - \beta^{\pi}(s,a))\right)^2 \right]   \nonumber  \\
    &=\E'\left[ \left(\frac{1}{M}\sum_{i=1}^M \left(\widehat{\beta}_{k,i}(s,a) - \beta^{\pi}(s,a)\right)\right)^2 \one[M>0] + \beta^\pi(s,a)^2 \one[M=0] \right]   \nonumber  \\
    &= \E'\left[ \left(\frac{1}{M}\sum_{i=1}^M \Delta_i-\delta(s,a) \right)^2 \one[M>0] + \beta^\pi(s,a)^2 \one[M=0] \right]  \nonumber  \\
    &\leq    \E'\left[ \left( 2\left(\frac{1}{M}\sum_{i=1}^M \Delta_i\right)^2  + 2\delta(s,a)^2 \right) \one[M>0] + \beta^\pi(s,a)^2 \one[M=0] \right]  \tag{using $(a-b)^2 \leq 2a^2+2b^2$}   \nonumber \\
    &\leq \Pr[w_1\leq B-N]\times \E'_{\{w_i\}}\left[ \E'\left[ 2\left(\frac{1}{M}\sum_{i=1}^M 
\Delta_i \right)^2 + 2\delta(s,a)^2~\Bigg\vert~ \{w_i\}\right] \Bigg\vert w_1\leq B-N \right] + \Pr[w_1 > B-N] \times (N+6\mix)^2 \tag{$\beta^\pi(s,a)\leq N+6\mix$ by Lemma~\ref{lemma: bounded span of pi}} \nonumber   \\
    &\leq \E'_{\{w_i\}}\left[ \E'\left[ 2\left(\frac{1}{M}\sum_{i=1}^M 
\Delta_i \right)^2 ~\Bigg\vert~ \{w_i\}\right] \Bigg\vert w_1\leq B-N \right] + \order\left(\frac{1}{T}\right)    \tag{using  Lemma~\ref{lemma: visit at least once}: $\Pr[w_1 > B-N]\leq \left(1-\frac{3}{4\hit}\right)^{\frac{B-N}{N}} \leq \frac{1}{T^3}$.}  \nonumber  \\
    &\leq \E'_{\{w_i\}}\left[  \frac{2}{M^2}\sum_{i=1}^M \E'\left[
\Delta_i^2 ~\big\vert~ \{w_i\}\right] \Bigg\vert w_1\leq B-N \right] + \order\left(\frac{1}{T}\right) \tag{$\Delta_i$ is zero-mean and independent of each other conditioned on $\{w_i\}$}  \nonumber  \\
    &\leq \E'_{\{w_i\}}\left[  \frac{2}{M^2} \cdot M\times \frac{\order(N^2)}{\pi(a|s)}\ \Bigg\vert w_1\leq B-N \right] + \order\left(\frac{1}{T}\right)  \tag{$\E'[\Delta_i^2]\leq \pi(a|s) \frac{\order(N^2)}{\pi(a|s)^2}=\frac{\order(N^2)}{\pi(a|s)}$ by definition of $\widehat{\beta}_k(s,a)$, Lemma~\ref{lemma: bounded span of pi}, and Eq.~\eqref{eqn: small error calculation}} \nonumber   \\
    &\leq \frac{\order(N^2)}{\pi(a|s)}\E'\left[\frac{1}{M}~\Big\vert~ w_1\leq B-N \right] + \order\left(\frac{1}{T}\right). \label{eqn: tmppp}
\end{align}
Since $\Pr'[M=0]\leq \frac{1}{T^3}$ by Lemma~\ref{lemma: visit at least once}, we have $\Pr'[w_1\leq B-N] = \Pr'[M>0]\geq 1-\frac{1}{T^3}$. Also note that if 
\begin{align*}
     M<M_0:=  \frac{B-N}{ 2N+\frac{4N\log T}{\mu^{\pi}(s)}}, 
\end{align*}
then there exists at least one waiting interval (i.e., $w_i$) longer than $\frac{4N\log T}{\mu^{\pi}(s)}$ (see Figure~\ref{figure: illustrate estimateQ} or \ref{figure: illustrate estimateQ reset}) . By Lemma~\ref{lemma: visit at least once}, this happens with probability smaller than $\left( 1-\frac{3\mu^{\pi}(s)}{4} \right)^{\frac{4\log T}{\mu^\pi(s)}}\leq \frac{1}{T^3}$.

Therefore, 
\begin{align*}
     \E'\left[\frac{1}{M}~\Big\vert~ M>0\right] 
     &= \frac{\sum_{m=1}^\infty \frac{1}{m}\Pr'[M=m] }{\Pr'[M>0]} 
     \leq \frac{1\times \Pr'[M<M_0] + \frac{1}{M_0}\times \Pr'[M\geq M_0]}{\Pr'[M>0]} \\
     &\leq \frac{1\times \frac{1}{T^3} + \frac{2N+\frac{4N\log T}{\mu^\pi(s)}}{B-N}}{1-\frac{1}{T^3}}\leq \order\left(\frac{N\log T}{B\mu^\pi(s)}\right). 
\end{align*}
Combining with \eqref{eqn: tmppp}, we get 
\begin{align*}
    \E'\left[ \left(\widehat{\beta}_k(s,a) - \beta^{\pi}(s,a))\right)^2 \right] \leq \order\left(\frac{N^3\log T}{B\pi(a|s)\mu^{\pi}(s)}\right). 
\end{align*}

\paragraph{Step 2.} By the same argument as in the ``Step 2'' of the previous proof for Eq.~\eqref{eqn: unbiased mean}, we have 
\begin{align*}
    \E\left[ \left(\widehat{\beta}_k(s,a) - \beta^{\pi}(s,a))\right)^2 \right] \leq \left(1+\frac{2}{T^2}\right) \E'\left[ \left(\widehat{\beta}_k(s,a) - \beta^{\pi}(s,a))\right)^2 \right] \leq  \order\left(\frac{N^3\log T}{B\pi(a|s)\mu^{\pi}(s)}\right),
\end{align*} 
which finishes the proof.
\end{proof}

\subsection{Proof for Lemma~\ref{lemma: stability lemma}}
\begin{proof}
     We defer the proof of Eq.~\eqref{eqn: OMD eta stable} to Lemma~\ref{lem:stability} and prove the rest of the statements assuming Eq.~\eqref{eqn: OMD eta stable}. 
     First, we have
\begin{align}
     \left\lvert J^{\pi_k}-J^{\pi_{k-1}}\right\rvert \nonumber 
     &= \left\lvert \sum_{s} \sum_{a} \mu^{\pi_k}(s)\left(\pi_k(a|s)-\pi_{k-1}(a|s)\right) q^{\pi_{k-1}}(s,a) \right\rvert  \tag{By Lemma~\ref{lemma: reward diff lemma}} \nonumber \\
     &\leq  \sum_{s} \sum_{a} \mu^{\pi_k}(s)\left\lvert\left(\pi_k(a|s)-\pi_{k-1}(a|s)\right)\right\rvert \left\lvert q^{\pi_{k-1}}(s,a)\right\rvert   \nonumber \\
     &= \order\left(\sum_{s} \sum_{a} \mu^{\pi_k}(s) N\eta \pi_{k-1}(a|s) \mix \right) \tag{By Eq.~\eqref{eqn: OMD eta stable} and Lemma~\ref{lemma: bounded span of pi}} \nonumber \\ 
     &= \order\left(\eta \mix N\right) = \order(\eta N^2).   \label{eqn: tmp stability 1}
\end{align}
Next, to prove a bound on $\left\lvert v^{\pi_{k}}(s)-v^{\pi_{k-1}}(s) \right\rvert$, first
note that for any policy $\pi$, 
\begin{align*}
    v^{\pi}(s)
    &= \sum_{n=0}^\infty \left(\mathbf{e}_s^\top (P^{\pi})^n - (\mu^{\pi})^\top \right)  r^\pi  \tag{By Eq.~\eqref{eqn:v_pi_identity}} \\
    &= \sum_{n=0}^{N-1} \left(\mathbf{e}_s^\top (P^{\pi})^n - (\mu^{\pi})^\top \right)   r^\pi   +  \sum_{n=N}^{\infty}\left(\mathbf{e}_s^\top (P^{\pi})^n - (\mu^{\pi})^\top \right)   r^\pi \\
     &= \sum_{n=0}^{N-1} \mathbf{e}_s^\top (P^{\pi})^n r^{\pi} -N J^\pi  + \text{error}^\pi(s), \tag{$J^\pi = (\mu^\pi)^\top r^\pi$}
\end{align*}
where $\text{error}^\pi(s):=\sum_{n=N}^{\infty}\left(\mathbf{e}_s^\top (P^{\pi})^n - \mu^{\pi} \right)^\top  r^\pi$. By Corollary~\ref{corollary: tail sum mixing time}, $|\text{error}^\pi(s)|\leq \frac{1}{T^2}$. 
Thus 
\begin{align}
    \left\lvert v^{\pi_{k}}(s)-v^{\pi_{k-1}}(s) \right\rvert 
&= \left\lvert \sum_{n=0}^{N-1} \mathbf{e}_s^\top \left((P^{\pi_k})^n-(P^{\pi_{k-1}})^n\right) r^{\pi_k} + \sum_{n=0}^{N-1} \mathbf{e}_s^\top (P^{\pi_{k-1}})^n (r^{\pi_k}-r^{\pi_{k-1}}) -NJ^{\pi_{k}}+ NJ^{\pi_{k-1}}\right\rvert + \frac{2}{T^2}  \nonumber \\
&\leq \sum_{n=0}^{N-1} \left\| \left((P^{\pi_k})^n-(P^{\pi_{k-1}})^n\right)r^{\pi_k}  \right\|_\infty  + \sum_{n=0}^{N-1} \|r^{\pi_k}-r^{\pi_{k-1}}\|_\infty + N\left\vert J^{\pi_{k}}-J^{\pi_{k-1}}\right\vert + \frac{2}{T^2}. \label{eqn: v decomposition}
\end{align}
Below we bound each individual term above (using notation $\pi':=\pi_k$, $\pi:=\pi_{k-1}$, $P':=P^{\pi_k}, P:=P^{\pi_{k-1}}$, $r':=r^{\pi_k}, r:=r^{\pi_{k-1}}$, $\mu:=\mu^{\pi_{k-1}}$ for simplicity). 
The first term can be bounded as
    \begin{align*}
        &\| (P'^n - P^n)r'\|_\infty \\
        &= \| \left( P'(P'^{n-1}-P^{n-1}) + (P'-P)P^{n-1} \right) r'\|_\infty \\
        &\leq \|  P'(P'^{n-1}-P^{n-1})r'\|_\infty + \| (P'-P)P^{n-1}  r'\|_\infty \\
        &\leq \|(P'^{n-1}-P^{n-1})r'\|_\infty + \| (P'-P)P^{n-1}  r'\|_\infty \tag{because every row of $P'$ sums to $1$} \\
        &= \|(P'^{n-1}-P^{n-1})r'\|_\infty + \max_{s} \left\lvert \mathbf{e}_s^\top (P'-P)P^{n-1}r' \right\lvert \\
        &\leq \|(P'^{n-1}-P^{n-1})r'\|_\infty + \max_{s} \| \mathbf{e}_s^\top (P'-P)P^{n-1}\|_1,
    \end{align*}
where the last term can be further bounded by
    \begin{align*}
        \max_{s} \| \mathbf{e}_s^\top (P'-P)P^{n-1}\|_1  
        &\leq \max_{s} \| \mathbf{e}_s^\top (P'-P)\|_1 \\ 
        &= \max_s \left(  \sum_{s'} \left\lvert\sum_{a}(\pi'(a|s)-\pi(a|s))p(s'|s,a) \right\rvert \right)  \\
        &\leq  \order\left(\max_s \left(  \sum_{s'}\sum_{a} \eta N \pi(a|s)p(s'|s,a) \right)\right) \tag{By Eq.~\eqref{eqn: OMD eta stable}}\\
        &=  \order\left(\eta N\right).  
    \end{align*}
    Repeatedly applying this bound we arrive at $\| (P'^n - P^n)r'\|_\infty \leq  \order\left(\eta N^2\right)$, 
and therefore, 
    \begin{align*}
       \sum_{n=0}^{N-1} \left\| \left((P^{\pi_k})^n-(P^{\pi_{k-1}})^n\right)r^{\pi_k}  \right\|_\infty  \leq \order\left(\eta N^3\right). 
    \end{align*}
    The second term in Eq.~\eqref{eqn: v decomposition} can be bounded as (by Eq.~\eqref{eqn: OMD eta stable} again)
    \begin{align*}
           \sum_{n=0}^{N-1} \|r'-r\|_\infty = \sum_{n=0}^{N-1} \max_s\left\lvert \sum_a (\pi'(a|s)-\pi(a|s)) r(s,a) \right\rvert 
          \leq \order\left(\sum_{n=0}^{N-1} \max_s  \sum_a \eta N \pi(a|s) \right)
          =  \order\left(\eta N^2\right),
    \end{align*}
    and the third term in Eq.~\eqref{eqn: v decomposition} is bounded via the earlier proof (for bounding $\left\vert J^{\pi_{k}}-J^{\pi_{k-1}}\right\vert$):
    \begin{align*}
          N\left\vert J^{\pi_{k}}-J^{\pi_{k-1}}\right\vert = \order\left(\eta N^3\right).  \tag{Eq.\eqref{eqn: tmp stability 1}}
    \end{align*}    
    Plugging everything into Eq.\eqref{eqn: v decomposition}, we prove $\left\lvert v^{\pi_{k}}(s)-v^{\pi_{k-1}}(s) \right\rvert = \order\left(\eta N^3\right)$. 

Finally, it is straightforward to prove the rest of the two statements:
   \begin{align*}
          \left \lvert q^{\pi_k}(s,a) - q^{\pi_{k-1}}(s,a) \right\rvert 
          &= \left\lvert r(s,a) + \E_{s'\sim p(\cdot|s,a)}[v^{\pi_k}(s')] - r(s,a) - \E_{s'\sim p(\cdot|s,a)}[v^{\pi_{k-1}}(s')]   \right\rvert \\
          &= \left\lvert  \E_{s'\sim p(\cdot|s,a)}[v^{\pi_k}(s') -v^{\pi_{k-1}}(s') ]  \right\rvert = \order\left(\eta N^3\right). 
   \end{align*}
      \begin{align*}
          \left \lvert \beta^{\pi_k}(s,a) - \beta^{\pi_{k-1}}(s,a) \right\rvert 
          &\leq \left\lvert q^{\pi_k}(s,a) - q^{\pi_{k-1}}(s,a) \right\rvert  + N \left \lvert J^{\pi_k} - J^{\pi_{k-1}} \right\rvert = \order\left(\eta N^3\right). 
   \end{align*}
This completes the proof.   
\end{proof}


\section{Analyzing Optimistic Online Mirror Descent with Log-barrier Regularizer --- Proofs for Eq.\eqref{eqn: OMD eta stable}, Lemma~\ref{lemma: PG reduction to bandit}, and Lemma~\ref{lemma: PG regret main term}}
In this section, we derive the \textit{stability property} (Eq.\eqref{eqn: OMD eta stable}) and the regret bound (Lemma~\ref{lemma: PG reduction to bandit} and Lemma~\ref{lemma: PG regret main term}) for optimistic online mirror descent with the log-barrier regularizer. Most of the analysis is similar to that in \citep{wei2018more, bubeck2019improved}.  Since in our \pg algorithm, we run optimistic online mirror descent independently on each state, the analysis in this section only focuses on a specific state $s$. We simplify our notations using $\pi_k(\cdot):= \pi_k(\cdot|s), \pi_k'(\cdot):= \pi_k'(\cdot|s), \widehat{\beta}_k(\cdot):=\widehat{\beta}_k(s, \cdot)$ throughout the whole section. 

Our \pg algorithm is effectively running Algorithm~\ref{alg: OOMD extract} on each state.
We first verify that the condition in Line~\ref{line: R_k condition} of Algorithm~\ref{alg: OOMD extract} indeed holds in our \pg algorithm. Recall that in \estimate (Algorithm~\ref{alg: estimate}) we collect trajectories in every episode for every state. Suppose for episode $k$ and state $s$ it collects $M$ trajectories that start from time $\tau_1, \ldots, \tau_M$ and has total reward $R_1, \ldots, R_M$ respectively. Let $m_a=\sum_{i=1}^M \one[a_{\tau_i}=a]$, then we have $\sum_a m_a = M$. By our way of constructing $\widehat{\beta}_k(s,\cdot)$, we have 
\begin{align*}
     \widehat{\beta}_k(s,a) = \sum_{i=1}^M \frac{R_i \one[a_{\tau_i}=a]}{M\pi_k(a|s)}
\end{align*}
when $M>0$.  Thus we have $\sum_a \pi_k(a|s)\widehat{\beta}_k(s,a) = \sum_a \sum_{i=1}^M \frac{R_i \one[a_{\tau_i}=a]}{M}=\sum_{i=1}^M \frac{R_i}{M}\leq (N+1)$ because every $R_i$ is the total reward for an interval of length $N+1$. This verifies the condition in Line~\ref{line: R_k condition} for the case $M>0$. When $M=0$, \estimate sets $\widehat{\beta}(s, \cdot)$ to zero so the condition clearly still holds. 

\setcounter{AlgoLine}{0}
\begin{algorithm}[t]
\caption{Optimistic Online Mirror Descent (OOMD) with log-barrier regularizer}
\label{alg: OOMD extract}

\textbf{Define:} \\
\quad $C:=N+1$\\
\quad Regularizer $\psi(x) = \frac{1}{\eta}\sum_{a=1}^A \log \frac{1}{x(a)}, \text{\ for\ }x\in \mathbb{R}_{+}^A$\\
\quad Bregman divergence associated with $\psi$: \[D_\psi(x, x') = \psi(x)-\psi(x') - \langle \nabla \psi(x'), x-x' \rangle \]\\
\textbf{Initialization}: $\pi_1'=\pi_1=\frac{1}{A}\one$\\
\For{$k=1, \ldots, K$}{
\nl       Receive $\widehat{\beta}_k\in\mathbb{R}_+^A$ for which $ \sum_a \pi_k(a) \widehat{\beta}_k(a) \leq C $. \label{line: R_k condition}  \\
\nl       Update
       \begin{align*}
       \pi_{k+1}' &= \argmax_{\pi\in \Delta_A}\left\{ \langle \pi, \widehat{\beta}_k \rangle  - D_{\psi}(\pi, \pi_{k}')   \right\} \\
       \pi_{k+1} &= \argmax_{\pi\in \Delta_A}\left\{  \langle \pi, \widehat{\beta}_k \rangle  - D_{\psi}(\pi, \pi_{k+1}')  \right\}
       \end{align*}
}
\end{algorithm}

\subsection{The stability property of Algorithm~\ref{alg: OOMD extract} --- Proof of Eq.\eqref{eqn: OMD eta stable}}
The statement and the proofs of Lemmas~\ref{lem:stability} and \ref{lemma: auxiliary} are almost identical to those of Lemma 9 and 10 in \citep{bubeck2019improved}. 
\begin{lemma}
\label{lem:stability}
In Algorighm~\ref{alg: OOMD extract}, if $\eta \leq \frac{1}{270C}=\frac{1}{270(N+1)}$, 
then 
\begin{align*}
     \left\lvert \pi_{k+1}(a) - \pi_{k}(a)  \right\rvert \leq 120\eta C \pi_{k}(a).  
\end{align*}
\end{lemma}

To prove this lemma we make use of the following auxiliary result,
where we use the notation $\norm{a}_M = \sqrt{a^\top M a}$ for a vector $a \in \fR^A$ and a positive semi-definite matrix $M \in \fR^{A\times A}$.
\begin{lemma}
\label{lemma: auxiliary}
For some arbitrary $b_1, b_2 \in \fR^A$, $a_0 \in \Delta_A$ with $\eta \leq \frac{1}{270C}$,
define
\begin{align*}
\begin{cases}
a_1 = \argmin_{a\in \Delta_A} F_1(a), \quad \text{where\ } F_1(a)\triangleq \inn{a, b_1} + D_\psi(a, a_0), \\
a_2 = \argmin_{a\in \Delta_A} F_2(a), \quad \text{where\ } F_2(a)\triangleq \inn{a, b_2} + D_\psi(a, a_0). 
\end{cases}
\end{align*}
($\psi$ and $D_\psi$ are defined in Algorithm~\ref{alg: OOMD extract}). Then as long as $\|b_1-b_2\|_{\nabla^{-2} \psi(a_1)} \leq 12\sqrt{\eta }C$, we have for all $i\in[A]$, $|a_{2,i}-a_{1,i}|\leq 60\eta Ca_{1,i} $. 
\end{lemma}
\begin{proof}[Proof of Lemma~\ref{lemma: auxiliary}]
First, we prove $\|a_1-a_2\|_{\nabla^2 \psi(a_1)}\leq 60\sqrt{\eta}C$ by contradiction. Assume $\|a_1-a_2\|_{\nabla^2 \psi(a_1)} > 60\sqrt{\eta}C$. Then there exists some $a_2'$ lying in the line segment between $a_1$ and $a_2$ such that $\|a_1-a_2'\|_{\nabla^2 \psi(a_1)}=60\sqrt{\eta}C$. By Taylor's theorem, there exists $\overline{a}$ that lies in the line segment between $a_1$ and $a_2'$ such that 

\begin{align}
F_2(a_2') 
&= F_2(a_1) + \inn{\nabla F_2(a_1), a_2'-a_1} + \frac{1}{2}\|a_2'-a_1\|_{\nabla^2 F_2(\overline{a})}^2 \nonumber \\
&= F_2(a_1) + \inn{b_2-b_1, a_2'-a_1} + \inn{\nabla F_1(a_1), a_2'-a_1} + \frac{1}{2}\|a_2'-a_1\|_{\nabla^2 \psi(\overline{a})}^2  \nonumber \\
&\geq F_2(a_1) - \|b_2-b_1\|_{\nabla^{-2} \psi(a_1)} \|a_2'-a_1\|_{\nabla^2 \psi(a_1)} + \frac{1}{2}\|a_2'-a_1\|_{\nabla^2 \psi(\overline{a})}^2 \nonumber \\
&\geq F_2(a_1) - 12\sqrt{\eta}C\times 60\sqrt{\eta}C + \frac{1}{2}\|a_2'-a_1\|_{\nabla^2 \psi(\overline{a})}^2     \label{eqn: to contradiction}
\end{align}
where in the first inequality we use H\"{o}lder inequality and the first-order optimality condition $\inn{\nabla F_1(a_1), a_2'-a_1}\geq 0$, 
and in the last inequality we use the conditions
$\|b_1-b_2\|_{\nabla^{-2} \psi(a_1)} \leq 12 \sqrt{\eta}C$
and $\|a_1-a_2'\|_{\nabla^2 \psi(a_1)}=60\sqrt{\eta}C$.  
Note that $\nabla^2 \psi(x)$ is a diagonal matrix
and $\nabla^2 \psi(x)_{ii} = \frac{1}{\eta}\frac{1}{x_i^2}$. Therefore for any $i\in[A]$, 
\begin{align*}
60\sqrt{\eta}C =  \|a_2'-a_1\|_{\nabla^2 \psi(a_1)}
= \sqrt{\sum_{j=1}^A \frac{(a_{2,j}'-a_{1,j})^2}{\eta a_{1,j}^2}}
\geq \frac{|a_{2,i}'-a_{1,i}|}{\sqrt{\eta} a_{1,i}}
\end{align*}
 and thus $\frac{|a_{2,i}'-a_{1,i}|}{ a_{1,i}} \leq 60\eta C \leq \frac{2}{9}$, which implies $\max\left\{ \frac{a_{2,i}'}{a_{1,i}}, \frac{a_{1,i}}{a_{2,i}'} \right\}\leq \frac{9}{7}$. Thus the last term in \eqref{eqn: to contradiction} can be lower bounded by
\begin{align*}
 \|a_2'-a_1\|_{\nabla^2 \psi(\overline{a})}^2 
&= \frac{1}{\eta}\sum_{i=1}^A \frac{1}{\overline{a}_i^2}(a_{2,i}'-a_{1,i})^2\geq  \frac{1}{\eta} \left(\frac{7}{9}\right)^2 \sum_{i=1}^A \frac{1}{a_{1,i}^2}(a_{2,i}'-a_{1,i})^2 \\
&\geq 0.6 \|a_2'-a_1\|_{\nabla^2 \psi(a_1)}^2 = 0.6\times \left(60\sqrt{\eta} C\right)^2 = 2160\eta C^2 .  
\end{align*}
Combining with \eqref{eqn: to contradiction} gives
\begin{align*}
F_2(a_2')\geq F_2(a_1) - 720\eta C^2 + \frac{1}{2}\times 2160\eta C^2> F_2(a_1). 
\end{align*}
Recall that $a_2'$ is a point in the line segment between $a_1$ and $a_2$. By the convexity of $F_2$, the above inequality implies $F_2(a_1)<F_2(a_2)$, contradicting the optimality of $a_2$. 

Thus we conclude $\|a_1-a_2\|_{\nabla^2 \psi(a_1)} \leq 60\sqrt{\eta}C$. Since $\|a_1-a_2\|_{\nabla^2 \psi(a_1)} = \sqrt{\sum_{j=1}^A \frac{(a_{1,j}-a_{2,j})^2}{\eta a_{1,j}^2} }\geq \frac{|a_{2,i}-a_{1,i}|}{\sqrt{\eta} a_{1,i}}$ for all $i$, we get $ \frac{|a_{2,i}-a_{1,i}|}{\sqrt{\eta} a_{1,i}} \leq 60\sqrt{ \eta}C$, which implies $|a_{2,i}-a_{1,i}|\leq 60\eta C a_{1,i} $. 
\end{proof}

\begin{proof}[Proof of Lemma~\ref{lem:stability}]
We prove the following stability inequalities
\begin{align}
\left\lvert \pi_k(a) - \pi_{k+1}'(a)  \right\rvert  &\leq  60\eta C \pi_k(a),    \label{eqn: to_prove1}\\
\left\lvert \pi_{k+1}'(a) - \pi_{k+1}(a)  \right\rvert  &\leq  60\eta C \pi_k(a).   \label{eqn: to_prove2}  
\end{align}
Note that \eqref{eqn: to_prove1} and \eqref{eqn: to_prove2} imply
\begin{align}
     \left\lvert \pi_k(a) - \pi_{k+1}(a)  \right\rvert  &\leq 120 \eta C \pi_k(a),    \label{eqn: imply1}
\end{align}
which is the inequality we want to prove. 

We use induction on $k$ to prove \eqref{eqn: to_prove1} and \eqref{eqn: to_prove2}. Note that 
\eqref{eqn: to_prove1} implies 
\begin{align}
     \pi_{k+1}'(a)\leq \pi_k(a) + 60\eta C\pi_k(a)\leq  \pi_k(a) + \frac{60}{270}\pi_k(a) \leq 2\pi_k(a),  \label{eqn: imply2}
\end{align}
and \eqref{eqn: imply1} implies 
\begin{align}
      \pi_{k+1}(a) \leq \pi_{k}(a) + 120\eta C\pi_{k}(a)\leq  \pi_{k}(a) + \frac{120}{270}\pi_k(a) \leq 2\pi_k(a).  \label{eqn: imply3}
\end{align}
Thus, \eqref{eqn: imply2} and \eqref{eqn: imply3} are also inequalities we may use in the induction process. 

\paragraph{Base case. } 
For the case $k=1$, note that 
\begin{align*}
      \begin{cases}
          \pi_1 = \argmin_{\pi \in \Delta_A} D_\psi(\pi ,\pi_1'),   \text{\qquad \qquad (because $\pi_1=\pi_1'$)}   \\
          \pi_{2}' =  \argmin_{\pi \in \Delta_A} \inn{\pi, -\widehat{\beta}_1} + D_\psi(\pi,\pi_1').    
    \end{cases}
\end{align*}
To apply Lemma~\ref{lemma: auxiliary} and obtain \eqref{eqn: to_prove1}, we only need to show $\| \widehat{\beta}_1\|_{\nabla^{-2}\psi(\pi_1)} \leq 12\sqrt{\eta} C$.  Recall $\nabla^2 \psi(u)_{ii} = \frac{1}{\eta}\frac{1}{u_i^2}$ and $\nabla^{-2} \psi(u)_{ii} = \eta u_i^2$.  Thus, 
\begin{align*}
  \|\widehat{\beta}_1\|_{\nabla^{-2}\psi(\pi_1)}^2 
  &\leq \sum_{a=1}^A \eta \pi_{1}(a)^2 \widehat{\beta}_1(a)^2  \leq \eta C^2
\end{align*}
because $\sum_a \pi_1(a)^2\widehat{\beta}_1(a)^2 \leq \left( \sum_a \pi_1(a)\widehat{\beta}_1(a)\right)^2 \leq C^2$ by the condition in Line~\ref{line: R_k condition} of Algorithm~\ref{alg: OOMD extract}. This proves \eqref{eqn: to_prove1} for the base case. 

Now we prove \eqref{eqn: to_prove2} of the base case. Note that
\begin{align}
    \begin{cases}
          \pi_{2}' = \argmin_{\pi\in \Delta_A} D_\psi(\pi, \pi_{2}'), \\
          \pi_{2} =  \argmin_{\pi\in \Delta_A} \inner{\pi, -\widehat{\beta}_1} + D_\psi(\pi,\pi_{2}'). 
    \end{cases}
    \label{eqn: update rule 2}
\end{align}
Similarly, with the help of Lemma~\ref{lemma: auxiliary}, we only need to show $\| \widehat{\beta}_1\|_{\nabla^{-2}\psi(\pi_2')} \leq 12\sqrt{\eta} C$.  This can be verified by 
\begin{align*}
  \|\widehat{\beta}_1\|_{\nabla^{-2}\psi(\pi_2')}^2 \leq \sum_{a=1}^A \eta \pi_{2}'(a)^2 \widehat{\beta}_1(a)^2 
  \leq 4\sum_{a=1}^A \eta \pi_1(a)^2  \widehat{\beta}_1(a)^2
  \leq 4\eta C^2, 
\end{align*}
where the second inequality uses \eqref{eqn: imply2} for the base case (implied by \eqref{eqn: to_prove1} for the base case, which we just proved). 

\paragraph{Induction.}
Assume \eqref{eqn: to_prove1} and \eqref{eqn: to_prove2} hold before $k$. 
To prove \eqref{eqn: to_prove1}, observe that
\begin{align}
    \begin{cases}
          \pi_k = \argmin_{\pi \in \Delta_A} \inner{\pi, -\widehat{\beta}_{k-1}} + D_\psi(\pi ,\pi_k'),  \\
          \pi_{k+1}' =  \argmin_{\pi \in \Delta_A} \inn{\pi, -\widehat{\beta}_k} + D_\psi(\pi,\pi_k').    
    \end{cases}
    \label{eqn: update rule 1}
\end{align}
To apply Lemma~\ref{lemma: auxiliary} and obtain \eqref{eqn: to_prove1}, we only need to show $\| \widehat{\beta}_k - \widehat{\beta}_{k-1} \|_{\nabla^{-2}\psi(\pi_k)} \leq 12\sqrt{\eta} C$.  This can be verified by
\begin{align*}
\|\widehat{\beta}_{k}-\widehat{\beta}_{k-1}\|_{\nabla^{-2} \psi(\pi_k)}^2 
&\leq \sum_{a=1}^A \eta \pi_k(a)^2 \left( \widehat{\beta}_{k}(a) - \widehat{\beta}_{k-1}(a)\right)^2 \\
&\leq 2\eta \sum_{a=1}^A \pi_k(a)^2\left( \widehat{\beta}_{k}(a)^2  + \widehat{\beta}_{k-1}(a)^2 \right) \\
&\leq 2\eta \sum_{a=1}^A \pi_k(a)^2 \widehat{\beta}_{k}(a)^2  + 2\eta \sum_{a=1}^A 4\pi_{k-1}(a)^2 \widehat{\beta}_{k-1}(a)^2   \\
&\leq 10\eta C^2, 
\end{align*}
where the third inequality uses \eqref{eqn: imply3} for $k-1$. 

To prove \eqref{eqn: to_prove2}, we observe: 
\begin{align}
    \begin{cases}
          \pi_{k+1}' = \argmin_{\pi\in \Delta_A} D_\psi(\pi, \pi_{k+1}'), \\
          \pi_{k+1} =  \argmin_{\pi\in \Delta_A} \inner{\pi, -\widehat{\beta}_k} + D_\psi(\pi,\pi_{k+1}'). 
    \end{cases}
    \label{eqn: update rule 2}
\end{align}
Similarly, with the help of Lemma~\ref{lemma: auxiliary}, we only need to show $\|\widehat{\beta}_k\|_{\nabla^{-2}\psi(\pi_{k+1}')}\leq 12\sqrt{\eta}C$. This can be verified by 
\begin{align*}
  \|\widehat{\beta}_k\|_{\nabla^{-2}\psi(\pi_{k+1}')}^2 
  \leq \sum_{a=1}^A \eta \pi_{k+1}'(a)^2 \widehat{\beta}_k(a)^2  
  \leq 4\sum_{a=1}^A \eta \pi_k(a)^2\widehat{\beta}_k(a)^2
  \leq 4\eta C^2, 
\end{align*}
where in the second inequality we use \eqref{eqn: imply2} (implied by \eqref{eqn: to_prove1}, which we just proved). This finishes the proof.
\end{proof}

\subsection{The regret bound of Algorithm~\ref{alg: OOMD extract} --- Proof of Lemma~\ref{lemma: PG reduction to bandit}}
\begin{proof}[Proof of Lemma~\ref{lemma: PG reduction to bandit}]
     By standard analysis for optimistic online mirror descent (e.g, \citep[Lemma 6]{wei2018more}, \citep[Lemma 5]{chiang2012online}), we have (recall $\widehat{\beta}_0$ is the all-zero vector)
\begin{align}
    \langle \tilde{\pi} - \pi_k, \widehat{\beta}_k \rangle \leq D_{\psi}(\tilde{\pi}, \pi_{k}') - D_{\psi}(\tilde{\pi}, \pi_{k+1}') + \langle  \pi_k-\pi_{k+1}', \widehat{\beta}_{k-1} -\widehat{\beta}_k \rangle  \label{eqn: regret analysis decomposition} 
\end{align}
      for any $\tilde{\pi}\in\Delta_A$. Summing over $k$ and telescoping give
\begin{align*}
     \sum_{k=1}^K \langle \tilde{\pi} - \pi_k, \widehat{\beta}_k \rangle \leq D_\psi(\tilde{\pi}, \pi_1') - D_\psi(\tilde{\pi}, \pi_{K+1}') + \sum_{k=1}^K  \langle  \pi_k-\pi_{k+1}', \widehat{\beta}_{k-1} -\widehat{\beta}_k \rangle
     \leq D_\psi(\tilde{\pi}, \pi_1') + \sum_{k=1}^K  \langle  \pi_k-\pi_{k+1}', \widehat{\beta}_{k-1} -\widehat{\beta}_k \rangle.   
\end{align*}
As in \citep{wei2018more}, we pick $\tilde{\pi} = \left(1-\frac{1}{T}\right)\pi^* + \frac{1}{TA}\one_A$, and thus
\begin{align*}
      D_\psi(\tilde{\pi}, \pi_1') 
      &= \psi(\tilde{\pi}) - \psi(\pi_1') - \langle \nabla\psi(\pi_1'), \tilde{\pi}-\pi_1' \rangle \\
      &= \psi(\tilde{\pi}) - \psi(\pi_1') \tag{$\nabla\psi(\pi_1')=-\frac{A}{\eta}\one$ and $\langle \one, \tilde{\pi}-\pi_1' \rangle=0$}\\
      &= \frac{1}{\eta}\sum_{a=1}^A  \log\frac{1}{\tilde{\pi}(a)} - \frac{1}{\eta}\sum_{a=1}^A  \log\frac{1}{\pi_1'(a)} \\
      &\leq \frac{A\log(AT)}{\eta}  - \frac{A\log A}{\eta} = \frac{A\ln T}{\eta}. 
\end{align*}

On the other hand, to bound $\langle \pi_k - \pi_{k+1}',  \widehat{\beta}_{k-1} - \widehat{\beta}_k  \rangle $, we follow the same approach as in  \citep[Lemma 14]{wei2018more}: define $F_k(\pi) = \langle \pi, -\widehat{\beta}_{k-1} \rangle + D_\psi(\pi, \pi_k')$ and $F_{k+1}'(\pi)=\langle \pi, -\widehat{\beta}_k  \rangle  + D_\psi(\pi, \pi_k')$.  Then by definition we have $\pi_k= \argmin_{\pi\in\Delta_A}F_k(\pi)$ and $\pi_{k+1}'=\argmin_{\pi\in\Delta_A}F_{t+1}'(\pi)$. 

Observe that 
\begin{align}
     F_{k+1}'(\pi_k) - F_{k+1}'(\pi_{k+1}')  \nonumber 
     &= (\pi_k-\pi_{k+1}')^\top (\widehat{\beta}_{k-1}-\widehat{\beta}_k) + F_k(\pi_k) - F_k(\pi_{k+1}')   \nonumber \\
     &\leq (\pi_k-\pi_{k+1}')^\top (\widehat{\beta}_{k-1}-\widehat{\beta}_k)   \tag{by the optimality of $\pi_k$} \nonumber \\
     &\leq \left\|  \pi_k - \pi_{k+1}'  \right\|_{\nabla^2\psi(\pi_k)}   \left\|  \widehat{\beta}_{k-1}-\widehat{\beta}_k  \right\|_{\nabla^{-2}\psi(\pi_k)}.  \label{eqn: two way upper bound}
\end{align}

On the other hand, for some $\xi$ that lies on the line segment between $\pi_k$ and $\pi_{k+1}'$, we have by Taylor's theorem and the optimality of $\pi_{k+1}'$, 
\begin{align}
      F_{k+1}'(\pi_k) - F_{k+1}'(\pi_{k+1}')  \nonumber 
      &= \nabla F_{k+1}'(\pi_{k+1}')^\top (\pi_k-\pi_{k+1}') + \frac{1}{2}\left\| \pi_k - \pi_{k+1}'  \right\|_{\nabla^2 F_{k+1}'(\xi)}^2  \nonumber \\
      &\geq \frac{1}{2}\left\| \pi_k - \pi_{k+1}'  \right\|_{\nabla^2 \psi(\xi)}^2   \tag{by the optimality of $\pi_{k+1}'$ and that $\nabla^2 F_{k+1}' = \nabla^2 \psi$} \\
       \label{eqn: two way lower bound}
\end{align}
By Eq.\eqref{eqn: to_prove1} we know $\pi_{k+1}'(a) \in \left[\frac{1}{2}\pi_k(a), 2\pi_{k}(a)\right]$, and hence $\xi(a)\in \left[\frac{1}{2}\pi_k(a), 2\pi_{k}(a)\right]$ holds as well, because $\xi$ is in the line segment between $\pi_k$ and $\pi_{k+1}'$. This implies for any $x$, 
\begin{align*}
     \|x\|_{\nabla^2\psi(\xi)} = \sqrt{\sum_{a=1}^A \frac{x(a)^2}{\eta \xi(a)^2}}\geq  \frac{1}{2}\sqrt{\sum_{a=1}^A \frac{x(a)^2}{\eta \pi_k(a)^2}}=\frac{1}{2}\|x\|_{\nabla^2\psi(\pi_k)}. 
\end{align*}
Combine this with \eqref{eqn: two way upper bound} and \eqref{eqn: two way lower bound}, we get 
\begin{align*}
      \left\|  \pi_k - \pi_{k+1}'  \right\|_{\nabla^2\psi(\pi_k)}   \left\|  \widehat{\beta}_{k-1}-\widehat{\beta}_k  \right\|_{\nabla^{-2}\psi(\pi_k)} \geq \frac{1}{8}\left\|  \pi_k - \pi_{k+1}'  \right\|_{\nabla^2\psi(\pi_k)}^2, 
\end{align*}
which implies $\left\|  \pi_k - \pi_{k+1}'  \right\|_{\nabla^2\psi(\pi_k)}\leq 8  \left\|  \widehat{\beta}_{k-1}-\widehat{\beta}_k  \right\|_{\nabla^{-2}\psi(\pi_k)}$. Hence we can bound the third term in \eqref{eqn: regret analysis decomposition} by 
\begin{align*}
      \left\|  \pi_k - \pi_{k+1}'  \right\|_{\nabla^2\psi(\pi_k)} \left\|  \widehat{\beta}_{k-1}-\widehat{\beta}_k  \right\|_{\nabla^{-2}\psi(\pi_k)} \leq 8 \left\|  \widehat{\beta}_{k-1}-\widehat{\beta}_k  \right\|_{\nabla^{-2}\psi(\pi_k)}^2 = 8\eta \sum_a \pi_k(a)^2 \left(\widehat{\beta}_{k-1}(a)-\widehat{\beta}_{k}(a)\right)^2. 
\end{align*}
Finally, combining everything we have 
\begin{align*}
     &\E\left[\sum_{k=1}^K   \langle \pi^* - \pi_k, \widehat{\beta}_k \rangle \right] \\
     &=\E\left[ \sum_{k=1}^K   \langle \pi^* - \tilde{\pi}, \widehat{\beta}_k \rangle +  \langle \tilde{\pi} - \pi_k, \widehat{\beta}_k \rangle \right] \\
     &\leq \left[ \frac{1}{T}\sum_{k=1}^K \left\langle \pi^* - \frac{1}{A}\one, \widehat{\beta}_k \right\rangle \right] + \order\left( \frac{A\log T}{\eta}+  \eta\sum_{k=1}^K \sum_a \pi_k(a)^2 \left(\widehat{\beta}_{k-1}(a)-\widehat{\beta}_{k}(a)\right)^2 \right), 
\end{align*}
where the expectation of the first term is bounded by $\order\left(\frac{KN}{T}\right) = \order(1)$ by the fact $\E[\widehat{\beta}_k(s)] = \order(N)$ (implied by Lemma~\ref{lemma: PG unbiasedness} and Lemma~\ref{lemma: bounded span of pi}).
This completes the proof.
\end{proof}

\subsection{Proof for Lemma~\ref{lemma: PG regret main term}}
\begin{lemma}[Restatement of Lemma~\ref{lemma: PG regret main term}]
     \begin{align}
     &\E\left[B\sum_{k=1}^K\sum_s\sum_a \mu^*(s)\left( \pi^*(a|s)-\pi_k(a|s) \right) q^{\pi_k}(s,a) \right] \nonumber \\
     &=  \otil\left(\frac{BA\ln T}{\eta} + \eta \frac{TN^3\ratio}{B} + \eta^3 TN^6\right).  \nonumber 
     \end{align}
     With the choice of $\eta = \min\left\{\frac{1}{270(N+1)}, \frac{B\sqrt{A}}{\sqrt{\ratio T N^3}}, \frac{\sqrt[4]{BA}}{\sqrt[4]{TN^6}} \right\}$, the bound becomes 
\begin{align*}
    \otil\left(\sqrt{N^3\ratio AT} + (BAN^2)^{\frac{3}{4}}T^{\frac{1}{4}} + BNA\right)  = \otil\left( \sqrt{\mix^3 \ratio AT} + (\mix^3\hit A)^{\frac{3}{4}}T^{\frac{1}{4}} + \mix^2\hit A\right). 
\end{align*}
\end{lemma}
\begin{proof}
    For any $s$, 
    \begin{align}
        &\E\left[\sum_{k=1}^K\sum_a (\pi^*(a|s)-\pi_k(a|s))q^{\pi_k}(s,a) \right] \nonumber \\
        &=\E\left[\sum_{k=1}^K\sum_a (\pi^*(a|s)-\pi_k(a|s))\beta^{\pi_k}(s,a) \right] \tag{by the definition of $\beta^{\pi_k}$ and that $\sum_{a}(\pi^*(a|s)-\pi_k(a|s))J^{\pi_k}=0$}\\
        &\leq \E\left[\sum_{k=1}^K\sum_a (\pi^*(a|s)-\pi_k(a|s)) \E_k \left[\widehat{\beta}_k(s,a)\right] \right] + \order \left(\frac{K}{T}\right) \tag{by Eq.~\eqref{eqn: unbiased mean}} \\
        &= \order\left(\frac{A \ln T}{\eta}\right) + \order \left(\eta \E\left[ \sum_{k=1}^K \sum_a \pi_k(a|s)^2 (\widehat{\beta}_k(s,a)-\widehat{\beta}_{k-1}(s,a))^2 \right]\right) \tag{by Lemma~\ref{lemma: PG reduction to bandit}}  \\
        &\leq \order \left(\frac{A \ln T}{\eta} + \eta N^2\right) +
        \order \left(\eta\E\left[ \sum_{k=2}^K \sum_a \pi_k(a|s)^2 (\widehat{\beta}_k(s,a)-\beta^{\pi_k}(s,a))^2 \right] \right) \nonumber  \\
        &\qquad + \order\left(\eta\E\left[ \sum_{k=2}^K \sum_a \pi_k(a|s)^2 ( \beta^{\pi_k}(s,a)-\beta^{\pi_{k-1}}(s,a))^2 \right] \right) \nonumber  \\
        &\qquad +  \order\left(\eta\E\left[ \sum_{k=2}^K \sum_a \pi_k(a|s)^2 (\beta^{\pi_{k-1}}(s,a)-\widehat{\beta}_{k-1}(s,a))^2 \right] \right),  \label{eqn: PG regret final decompose}
    \end{align}
    where the last line uses the fact $(z_1+z_2+z_3)^2\leq 3z_1^2+3z_2^2+3z_3^2$.
    The second term in \eqref{eqn: PG regret final decompose} can be bounded using Eq.~\eqref{eqn: unbiased variance}: 
    \begin{align*}
         &\order \left(\eta\E\left[ \sum_{k=2}^K \sum_a \pi_k(a|s)^2 (\widehat{\beta}_k(s,a)-\beta^{\pi_k}(s,a))^2 \right]\right) \\
         &=  \order\left(\eta\E\left[ \sum_{k=2}^K \sum_a \pi_k(a|s)^2 \frac{N^3\log T}{B\pi_k(a|s)\mu^{\pi_k}(s)} \right]  \right)\\
         &=  \order\left(\eta\E\left[ \sum_{k=2}^K  \frac{N^3 \log T}{B\mu^{\pi_k}(s)} \right] \right). 
    \end{align*}
    The fourth term in \eqref{eqn: PG regret final decompose} can be bounded similarly, except that we first use Lemma~\ref{lem:stability} to upper bound $\pi_k(a|s)$ by $2\pi_{k-1}(a|s)$. Eventually this term is upper bounded by $\order\left( \eta \E\left[ \sum_{k=2}^{K}  \frac{N^3 \log T}{B\mu^{\pi_{k-1}}(s)} \right] \right) = \order\left(  \eta\E\left[ \sum_{k=1}^{K}  \frac{N^3 \log T}{B\mu^{\pi_k}(s)} \right] \right)$. 

 The third term in \eqref{eqn: PG regret final decompose} can be bounded using Lemma~\ref{lemma: stability lemma}: 
\begin{align*}
    &\order\left( \eta\E\left[ \sum_{k=2}^K \sum_a \pi_k(a|s)^2 ( \beta^{\pi_k}(s,a)-\beta^{\pi_{k-1}}(s,a))^2 \right]\right) \\
    &= \order\left( \eta\E\left[ \sum_{k=2}^K \sum_a \pi_k(a|s)^2 (\eta N^3)^2 \right] \right) \\
    &= \order\left( \eta^3 KN^6\right). 
\end{align*}
Combining all these bounds in \eqref{eqn: PG regret final decompose}, we get 
\begin{align*}
     \E\left[\sum_{k=1}^K\sum_a (\pi^*(a|s)-\pi_k(a|s))q^{\pi_k}(s,a) \right] = \order\left( \frac{A\ln T}{\eta} + \eta \E\left[ \sum_{k=1}^K  \frac{N^3 \log T}{B\mu^{\pi_k}(s)} \right] + \eta^3 KN^6  \right). 
\end{align*}
Now multiplying both sides by $B\mu^*(s)$ and summing over $s$ we get 
\begin{align*}
      \E\left[B\sum_{k=1}^K\sum_s \sum_a \mu^*(s) (\pi^*(a|s)-\pi_k(a|s))q^{\pi_k}(s,a) \right] 
       &= \order\left( \frac{BA\ln T}{\eta} + \eta \E\left[ \sum_{k=1}^K\sum_s  \frac{N^3 (\log T) \mu^*(s) }{\mu^{\pi_k}(s)} \right] + \eta^3 BKN^6  \right) \\
      &\leq \order\left( \frac{BA\ln T}{\eta} +  \eta \ratio KN^3 (\log T)  + \eta^3 BKN^6  \right) \\
      &=  \otil\left( \frac{BA}{\eta} +  \eta \ratio \frac{TN^3}{B}  + \eta^3 TN^6  \right)   \tag{$T=BK$}
\end{align*}
Choosing $\eta = \min\left\{\frac{1}{270(N+1)}, \frac{B\sqrt{A}}{\sqrt{\ratio T N^3}}, \frac{\sqrt[4]{BA}}{\sqrt[4]{TN^6}} \right\}$ ($\eta\leq \frac{1}{270(N+1)}$ is required by Lemma~\ref{lem:stability}), we finally obtain
\begin{align*}
     \E\left[B\sum_{k=1}^K\sum_s \sum_a \mu^*(s) (\pi^*(a|s)-\pi_k(a|s))q^{\pi_k}(s,a) \right]  
     &= \otil\left(\sqrt{N^3\ratio AT} + (BAN^2)^{\frac{3}{4}}T^{\frac{1}{4}} + BNA\right) \\
     &= \otil\left( \sqrt{\mix^3 \ratio AT} + (\mix^3\hit A)^{\frac{3}{4}}T^{\frac{1}{4}} + \mix^2\hit A\right). 
\end{align*}
\end{proof}


\section{Experiments}
\label{app:experiments}
In this section, we compare the performance of our proposed algorithms and previous model-free algorithms. We note that model-based algorithms (UCRL2, PSRL, \dots) typically have better performance in terms of regret but require more memory. For a fair comparison, we restrict our attention to model-free algorithms.

Two environments are considered: a randomly generated MDP and  JumpRiverSwim. Both of the environments consist of 6 states and 2 actions. The reward function and the transition kernel of the random MDP are chosen uniformly at random. The JumpRiverSwim environment is a modification of the RiverSwim environment \cite{strehl2008analysis,ouyang2017learningbased} with a small probability of jumping to an arbitrary state at each time step. 

The standard RiverSwim models a swimmer who can choose to swim either left or right in a river. The states are arranged in a chain and the swimmer starts from the leftmost state ($s = 1$). If the swimmer chooses to swim left, i.e., the direction of the river current, he is always successful. If he chooses to swim right, he may fail with a certain probability. The reward function is: $r(1, \text{left}) = 0.2$, $r(6, \text{right}) = 1$ and $r(s, a) = 0$ for all other states and actions. The optimal policy is to always swim right to gain the maximum reward of state $s = 6$. The standard RiverSwim is not an ergodic MDP and does not satisfy the assumption of the MDP-OOMD algorithm. To handle this issue, we consider the JumpRiverSwim environment which has a small probability $0.01$ of moving to an arbitrary state at each time step. This small modification provides an ergodic environment.

\begin{figure}
\begin{center}
\includegraphics[width=0.45\textwidth]{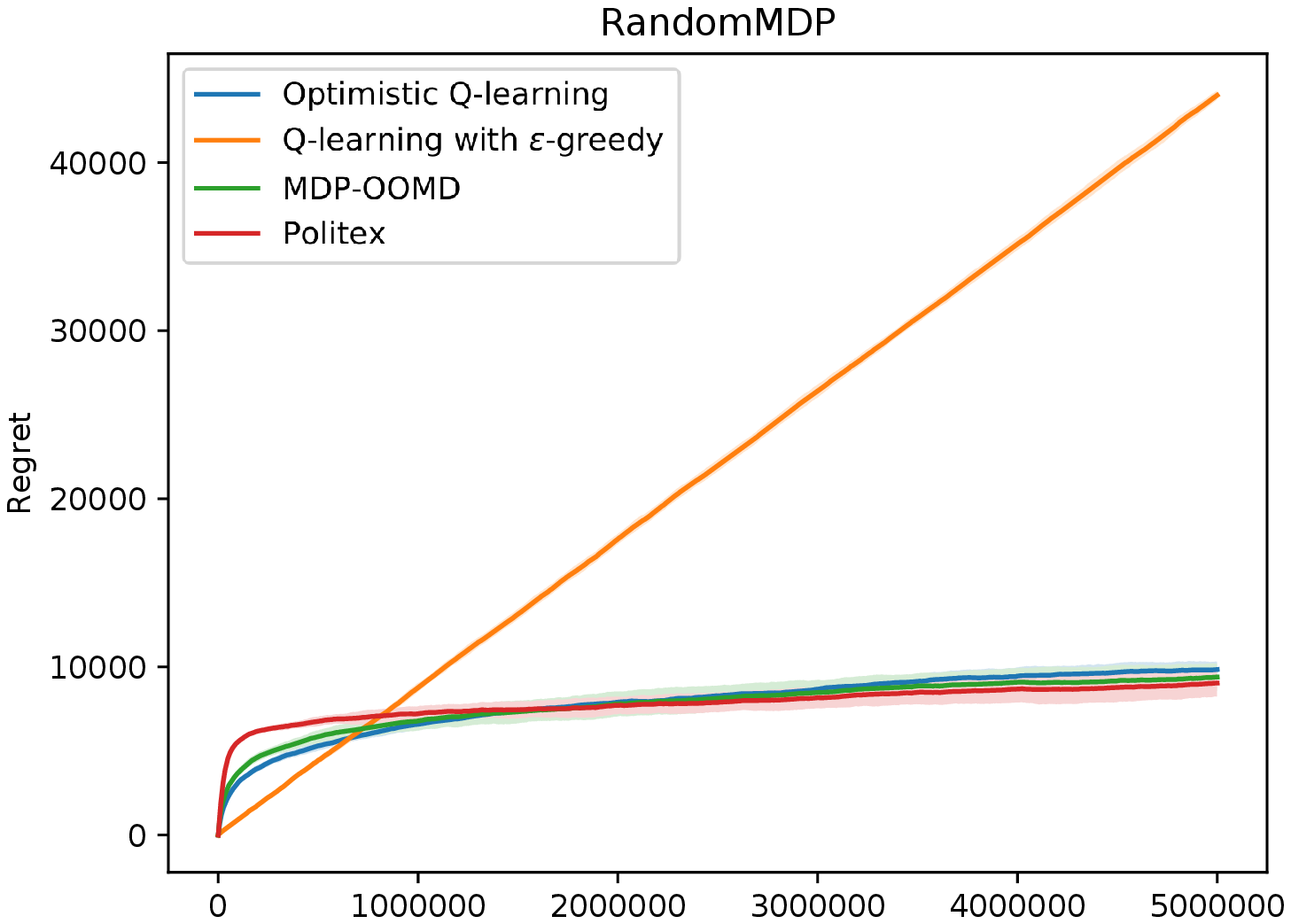}
\includegraphics[width=0.45\textwidth]{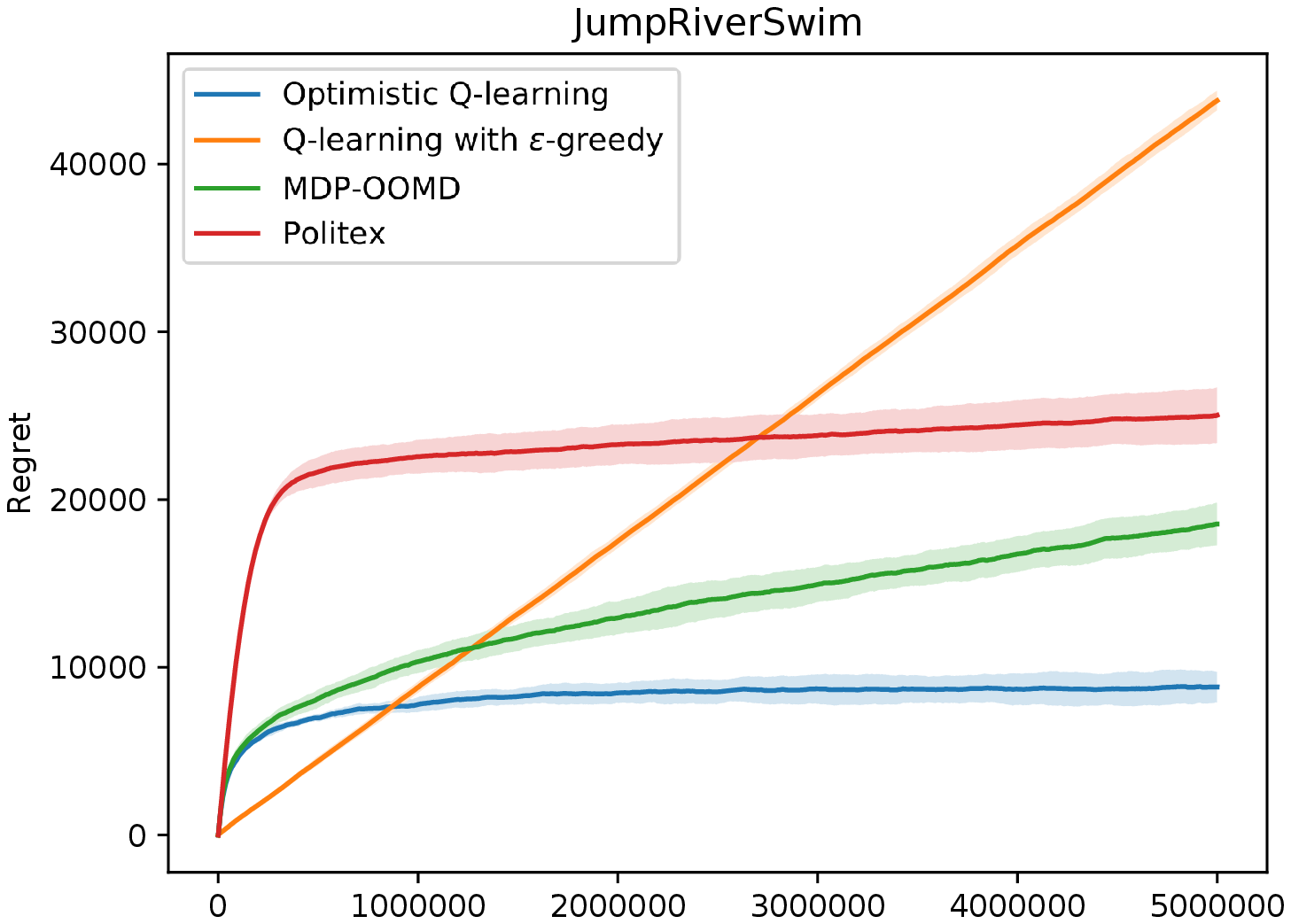}
\caption{Performance of model-free algorithms on random MDP (left) and JumpRiverSwim (right). The standard Q-learning algorithm with $\epsilon$-greedy exploration suffers from linear regret. The \textsc{Optimistic Q-learning} and MDP-OOMD algorithms achieve sub-linear regret. The shaded area denotes the standard deviation of regret over multiple runs.}
\label{figure: experiment}
\end{center}
\end{figure}

We compare our algorithms with two benchmark model-free algorithms. The first benchmark is the standard Q-learning with $\epsilon$-greedy exploration. Figure~\ref{figure: experiment} shows that this algorithm suffers from linear regret, indicating that the naive $\epsilon$-greedy exploration is not efficient. The second benchmark is the \politex algorithm by~\citet{abbasi2019politex}. The implementation of \politex is based on the variant designed for the tabular case, which is presented in their Appendix F and Figure 3. \politex usually requires longer episode length than MDP-OOMD (see Table \ref{tab: hyper parameters}) because in each episode it needs to accurately estimate the Q-function, rather than merely getting an unbiased estimator of it as in MDP-OOMD. 
Figure \ref{figure: experiment} shows that the proposed \textsc{Optimistic Q-learning}, MDP-OOMD algorithms, and the \politex algorithm by~\citet{abbasi2019politex} all achieve similar performance in the RandomMDP environment.  In the JumpRiverSwim environment, the Optimistic Q-learning algorithm outperforms the other three algorithms. Although the regret upper bound for \textsc{Optimistic Q-learning} scales as $\otil(T^{2/3})$ (Theorem \ref{thm: q-learning}), which is worse than that of MDP-OOMD (Theorem \ref{thm: policy optimization thm}), Figure \ref{figure: experiment} suggests that in the environments that lack good mixing properties, \textsc{Optimistic Q-learning} algorithm may perform better. The detail of the experiments is listed in Table \ref{tab: hyper parameters}.

\begin{table*}[t]
\caption{Hyper parameters used in the experiments. These hyper parameters are optimized to perform the best possible result for all the algorithms. All the experiments are averaged over 10 independent runs for a horizon of $5 \times 10^6$.  For the \politex algorithm, $\tau$ and $\tau'$ are the lengths of the two stages defined in Figure 3 of \cite{abbasi2019politex}. }
\label{tab: hyper parameters}
\begin{center}
\renewcommand{\arraystretch}{1.3}
\begin{tabular}{ |c|l|l| }
\hline
 & \textbf{Algorithm} & \textbf{Parameters} \\ \hline
\multirow{4}{*}{Random MDP} 
 & Q-learning with $\epsilon$-greedy & $\epsilon=0.05$  \\
 & Optimistic Q-learning & $H = 100, c = 1$, $b_\tau = c \sqrt{H/\tau}$ \\
 & MDP-OOMD & $N = 2, B = 4, \eta=0.01$  \\ 
 & \politex &  $\tau=1000$, $\tau'=1000, \eta=0.2$  \\ \hline
 \hline
\multirow{4}{*}{JumpRiverSwim} 
 & Q-learning with $\epsilon$-greedy & $\epsilon=0.03$  \\
 & Optimistic Q-learning & $H = 100, c = 1$, $b_\tau = c \sqrt{H/\tau}$ \\
 & MDP-OOMD & $N = 10, B = 30, \eta=0.01$  \\ 
& \politex &  $\tau=3000$, $\tau'=3000, \eta=0.2$  \\ \hline
\end{tabular}
\end{center}
\end{table*}

\end{document}